\documentclass[journal]{IEEEtran}
\usepackage{cite}
\usepackage{amsmath,amssymb,amsfonts}
\usepackage{graphicx}
\usepackage{textcomp}
\usepackage{csquotes}
\usepackage{xcolor,soul}
\usepackage{adjustbox}
\usepackage{hyperref}
\usepackage{multirow}
\usepackage{titlesec}
\usepackage{subcaption}
\usepackage{makecell}
\usepackage{orcidlink}
\usepackage{booktabs} 
\usepackage{caption} 
\usepackage[numbers]{natbib}
\usepackage{cancel}

\usepackage{etoolbox}
\usepackage{adjustbox}
\usepackage{amsthm}
\newtheorem{proposition}{Proposition}
\usepackage{siunitx}
\usepackage{algorithm}
\usepackage{algpseudocode}

\definecolor{columbiablue}{rgb}{0.61, 0.87, 1.0}
\titlespacing{\subsubsection}{0pt}{0pt}{1pt}

\begin{document}

\title{Bayesian Learning-Enhanced Navigation with Deep Smoothing for Inertial-Aided Navigation}

\author{Nadav~Cohen \orcidlink{0000-0002-8249-0239}
        and~Itzik~Klein \orcidlink{0000-0001-7846-0654}
\thanks{N. Cohen and I. Klein are with the Hatter Department of Marine Technologies, Charney School of Marine Sciences, University of Haifa, Israel.\\ Corresponding author: ncohe140@campus.haifa.ac.il (N.Cohen)}
}

\maketitle

\begin{abstract}
Accurate post-processing navigation is essential for applications such as survey and mapping, where the full measurement history can be exploited to refine past state estimates. Fixed-interval  smoothing algorithms represent the theoretically optimal solution  under Gaussian assumptions. However, loosely coupled INS/GNSS systems fundamentally inherit the systematic position bias of raw GNSS measurements, leaving a persistent accuracy gap that model-based smoothers cannot resolve.  To address this limitation, we propose BLENDS, which integrates Bayesian learning with deep smoothing to enhance navigation performance. BLENDS is a a data-driven post-processing 
framework that augments the classical two-filter smoother with a transformer-based neural network. It learns to modify the filter covariance matrices and apply an additive correction to the smoothed  error-state directly within the Bayesian framework. A novel Bayesian-consistent loss jointly supervises the smoothed mean and 
covariance, enforcing minimum-variance estimates while maintaining  statistical consistency. BLENDS is evaluated on two real-world  datasets spanning a mobile robot and a quadrotor. Across all  unseen test trajectories, BLENDS achieves horizontal position  improvements of up to 63\% over the baseline forward EKF.
\end{abstract}

\begin{IEEEkeywords}
Inertial Navigation System, GNSS, Extended Kalman Smoothing, Two-Filter Smoother, Deep Learning, Sensor Fusion
\end{IEEEkeywords}

\IEEEpeerreviewmaketitle

\section{Introduction}
\noindent
\IEEEPARstart{A}{ccurate} and reliable navigation is a fundamental requirement for 
autonomous mobile platforms operating in outdoor environments. 
Inertial navigation systems (INS), driven by inertial measurement 
units (IMUs), provide high-rate, self-contained state estimates but 
suffer from unbounded error growth due to the integration of 
accelerometer and gyroscope measurements, which are corrupted by 
noise and bias~\cite{chatfield1997fundamentals}. Global 
navigation satellite systems (GNSS) provide global position 
measurements that bound this drift, and their fusion with INS through 
an extended Kalman filter (EKF) has become the standard approach for navigation across ground vehicles, aerial platforms, and 
marine systems~\cite{titterton2004strapdown}. Despite its widespread adoption, loosely coupled INS/GNSS fusion is 
limited by the accuracy of the GNSS measurements, which 
typically achieves two to five meters of positional accuracy, far 
from the centimeter-level precision offered by real-time kinematics (RTK) corrected GNSS~\cite{teunissen2017springer}. This 
discrepancy manifests as a systematic position offset in the fused 
solution, which directly motivates the approach proposed in this paper.
\\ \noindent
In some cases, the navigation performance recorded during a mission 
is insufficient for post-processing analysis, for example in 
surveying and mapping~\cite{shin2005backward}, airborne 
gravimetry~\cite{kwon2001new}, and direct georeferencing of 
unmanned aerial vehicle imagery~\cite{chiang2012development}. To this end, prior work has proposed exploiting the full measurement history, incorporating both past and future observations, to refine the navigation solution. Fixed-interval smoothing algorithms, such as 
the Rauch--Tung--Striebel smoother (RTSS)~\cite{rauch1965maximum} 
and the two-filter smoother (TFS)~\cite{fraser1969optimum}, represent 
the theoretically optimal solution for post-processing navigation 
under Gaussian assumptions. Beyond refining positioning performance, 
these algorithms have also been applied to address challenges arising 
in real-time scenarios, such as GNSS-denied environments, where they 
demonstrate the ability to compensate for EKF 
divergence~\cite{liu2010two,zhang2017new}, and in conditions where 
the GNSS signal is severely disturbed or 
degraded~\cite{yin2023sensor,zhao2025smoothing}.
\\ \noindent
With recent advancements in hardware capabilities and computational 
power, data-driven approaches such as machine and deep learning have 
begun to be integrated into the filtering mechanism to enhance 
navigation performance. Specifically, in inertial-aided navigation, 
data-driven methods have been investigated either as end-to-end 
approaches or as replacements for specific blocks within the filtering 
framework~\cite{cohen2024inertial,chen2024deep}. Replacing a specific block within 
the framework, while maintaining the reliability and interpretability 
of the navigation solution rather than adopting a black-box approach, 
has shown great promise. Adaptive approaches have been suggested to estimate the 
process and measurement noise covariance 
matrices~\cite{wu2020predicting,cohen2025adaptive,levy2026adaptive,versano2026hybrid}. 
Data-driven methods have also been used to compensate for partial or 
intermittent measurement availability and subsequently integrated 
into the filter~\cite{fang2020lstm,cohen2024seamless}. These efforts have also extended to the post-processing smoothing 
stage. In~\cite{ni2022rtsnet,revach2023rtsnet}, a hybrid deep 
learning and Kalman smoothing network called RTSNet was proposed. 
However, this approach has not been evaluated on INS/GNSS fusion and 
is specifically designed for partially known state-space models, 
which is not the case in INS/GNSS fusion where the process and 
measurement noise covariance matrices can be determined from sensor 
specifications. Furthermore, RTSNet is designed to learn the forward 
and backward filter gains jointly, which does not explicitly enforce 
the minimum-variance optimality criterion~\cite{bar2001estimation}, 
nor does it address the systematic position bias that arises from raw GNSS measurements.
\\ \noindent
To address these limitations, this paper proposes BLENDS,
a Bayesian learning-enhanced deep smoothing framework for inertial-aided navigation.
BLENDS is a data-driven 
post-processing framework that augments the classical TFS with a transformer-based neural network. Unlike existing 
approaches, BLENDS operates directly within the Bayesian smoothing 
framework rather than replacing it, preserving the theoretical 
guarantees of the TFS while learning to correct the systematic 
position bias that classical smoothers are unable to address. The 
network learns to modify the filter covariance matrices and apply an 
additive correction to the smoothed error-state, guided by a 
Bayesian-consistent loss that jointly supervises the smoothed mean 
and covariance. BLENDS is evaluated on two real-world 
datasets spanning substantially different platforms, a mobile 
robot and a quadrotor. The main contributions of this paper are:
\begin{enumerate}
    \item A learning-enhanced two-filter smoother that corrects 
    systematic GNSS position bias within the Bayesian framework, 
    without compromising the theoretical optimality of the smoothed 
    solution.
    \item A Bayesian-consistent loss function that jointly supervises 
    the smoothed mean and covariance, enforcing minimum-variance 
    estimates while maintaining statistical consistency.
    \item A training strategy that constrains the network outputs to physically meaningful ranges, ensuring stable convergence across different platforms and dynamics.
\end{enumerate}

\noindent
The proposed method is validated on two real-world datasets: a 
55-minute mobile robot dataset covering six distinct motion patterns, and the INSANE quadrotor 
benchmark~\cite{brommer2022insane}. Across all unseen test 
trajectories, BLENDS achieves horizontal position improvements of up 
to 63\% over the forward EKF, while maintaining covariance 
consistency across both platforms.
\\ \noindent
The remainder of this paper is organized as follows. 
Section~\ref{sec:system} presents the system modeling and estimation 
framework, including the forward error-state EKF, the backward 
information filter, the TFS, and the RTSS. 
Section~\ref{sec:motivation} motivates the proposed approach through 
simulation, demonstrating the fundamental limitation of classical 
smoothers in the presence of biased GNSS measurements. 
Section~\ref{sec:proposed} describes the BLENDS framework, the 
network architecture, and the Bayesian-consistent loss. 
Section~\ref{sec:results} presents the datasets, evaluation metrics, 
implementation details, experimental results, and discussion. 
Section~\ref{sec:conclusion} concludes the paper.
\section{System Modeling and Estimation Framework}\label{sec:system}
\noindent
The following subsections introduce smoothing techniques for loosely coupled INS/GNSS sensor fusion. State smoothing aims to estimate the marginal posterior distribution of the state at time step $k$ using the entire set of measurements available up to a later time $T$, where $T > k$. In contrast to filtering, which relies only on measurements up to the current time, smoothing incorporates future observations in order to refine past state estimates. Fixed–interval smoothing is commonly implemented using either the TFS, also referred to as forward-backward smoothing (FBS), or the RTSS. The forward–backward approach provides the most straightforward implementation of fixed-interval smoothing. In contrast, the RTSS is conceptually more challenging but computationally more efficient than the TFS, while remaining theoretically equivalent in the linear case~\cite{simon2006optimal}.
\subsection{Forward Error-State Extended Kalman Filter}\label{EKF}
\noindent
The system state vector $\boldsymbol{x}$ is decomposed into a nominal state vector $\boldsymbol{x}^{NOM}$ 
and a perturbation error state vector $\delta\boldsymbol{x}$, as:
\begin{equation}\label{eqn:error_state}
\boldsymbol{x} = \boldsymbol{x}^{NOM} - \delta\boldsymbol{x}.
\end{equation}
The forward error-state vector is defined as follows:
\begin{equation}\label{eqn:dx}
\delta\boldsymbol{x}^n_{f} = \left[(\delta\boldsymbol{p}^{n})^{T} \quad (\delta\boldsymbol{v}^{n})^{T} \quad (\delta\boldsymbol{\epsilon}^{n})^{T} \quad \delta\boldsymbol{b}_{a}^{T} \quad \delta\boldsymbol{b}_{g}^{T}\right]^{T}
\end{equation}
where the subscript $f$ denotes the forward filter. Here, $\delta\boldsymbol{p}^{n} \in \mathbb{R}^{3\times1}$, $\delta\boldsymbol{v}^{n} \in \mathbb{R}^{3\times1}$, $\delta\boldsymbol{\epsilon}^{n} \in \mathbb{R}^{3\times1}$, $\delta\boldsymbol{b}_{a} \in \mathbb{R}^{3\times1}$, and $\delta\boldsymbol{b}_{g} \in \mathbb{R}^{3\times1}$ denote the position error expressed in the navigation frame (e.g., latitude, longitude, and altitude), the velocity error expressed in the navigation frame, the misalignment angles between predicted and true navigation frame, the accelerometer bias error, and the gyroscope bias error, respectively~\cite{farrell2008aided}. The governing state-space differential equation is:
\begin{equation}\label{eqn:f}
\delta\dot{\boldsymbol{x}}_{f} = \mathbf{F}\,\delta\boldsymbol{x}_{f} + \mathbf{G}\,\boldsymbol{\omega}
\end{equation}
where the system matrix $\mathbf{F} \in \mathbb{R}^{15\times15}$ is derived by linearization of the inertial 
equations of motion in the navigation frame, and the noise distribution matrix 
$\mathbf{G} \in \mathbb{R}^{15\times12}$ is constructed accordingly. The process noise vector, assuming 
the biases evolve as random walks, is defined as:
\begin{equation}\label{eqn:noise}
\boldsymbol{\omega} = \left[\boldsymbol{\omega}_{a}^{T} \quad \boldsymbol{\omega}_{g}^{T} \quad 
\boldsymbol{\omega}_{a_{b}}^{T} \quad \boldsymbol{\omega}_{g_{b}}^{T}\right]^{T} \in \mathbb{R}^{12\times1}
\end{equation}
where $\boldsymbol{\omega}_{a} \in \mathbb{R}^{3\times1}$ and $\boldsymbol{\omega}_{g} \in \mathbb{R}^{3\times1}$ 
represent the additive noise for the accelerometer and gyroscope, respectively, and 
$\boldsymbol{\omega}_{a_{b}} \in \mathbb{R}^{3\times1}$ and $\boldsymbol{\omega}_{g_{b}} \in \mathbb{R}^{3\times1}$ 
denote the random walk driving noise for the accelerometer and gyroscope biases, respectively, such that:
\begin{equation}\label{eqn:biasdot}
\dot{\boldsymbol{b}}_{a} = \boldsymbol{\omega}_{a_{b}}, \quad \dot{\boldsymbol{b}}_{g} = \boldsymbol{\omega}_{g_{b}}
\end{equation}
In discrete form, \eqref{eqn:f} is discretized for the prediction phase of the Kalman filter as:
\begin{equation}\label{eqn:pred_state}
\delta\boldsymbol{x}^{-}_{f,k} = \mathbf{\Phi}_{k-1}\,\delta\boldsymbol{x}^{+}_{f,k-1}
\end{equation}
\begin{equation}\label{eqn:pred}
\mathbf{P}^{-}_{f,k} = \mathbf{\Phi}_{k-1}\mathbf{P}^{+}_{f,k-1}\mathbf{\Phi}_{k-1}^{T} + \mathbf{Q}_{k-1}
\end{equation}
where $\mathbf{P} \in \mathbb{R}^{15\times15}$ is the error-state covariance matrix, and the superscripts 
$(\cdot)^{-}$ and $(\cdot)^{+}$ denote the a \textit{priori} (predicted) and a \textit{posteriori} (updated) estimates, 
respectively. The state transition matrix $\mathbf{\Phi}$ is approximated using a first-order 
power-series expansion of the system matrix over the time interval $dt$, where $\mathbf{I}_{15} \in 
\mathbb{R}^{15\times15}$ denotes the identity matrix:
\begin{equation}\label{eqn:transition_matrix}
\mathbf{\Phi} = \mathbf{I}_{15} + \mathbf{F}\,dt
\end{equation}
and the approximated discretized process noise covariance matrix of $\mathbf{Q} = \mathbb{E}[\boldsymbol{\omega}\boldsymbol{\omega}^{T}]$ 
is given by:
\begin{equation}\label{eqn:Q_dis_MB}
\mathbf{Q}_{k-1} = \frac{1}{2}\,\mathbf{G}_{k-1}\,\mathbf{Q}\,\mathbf{G}_{k-1}^{T}\,dt
\end{equation}
In a loosely coupled INS/GNSS fusion, a position update is performed using the following 
measurement model:
\begin{equation}\label{eqn:mes}
\delta\boldsymbol{z}_{k} = \mathbf{H}_k\,\delta\boldsymbol{x}^{-}_{f,k} + \boldsymbol{v}
\end{equation}
where $\delta\boldsymbol{z}_{k}$ denotes the measurement residual between the INS solution and 
the GNSS position measurement, given by:
\begin{equation}\label{eqn:meas_diff}
\delta\boldsymbol{z}_{k} = \boldsymbol{z}^{INS}_{k} - \boldsymbol{z}^{GNSS}_{k}
\end{equation}
and $\boldsymbol{v}$ is the measurement noise vector. The measurement matrix 
$\mathbf{H}_k \in \mathbb{R}^{3\times15}$ maps the error-state vector to the position 
measurement residual, and is defined as:
\begin{equation}\label{eqn:H}
\mathbf{H}_k = \left[\mathbf{I}_{3} \quad \mathbf{0}_{3\times12}\right]
\end{equation}
where $\mathbf{I}_{3} \in \mathbb{R}^{3\times3}$ is the identity matrix and 
$\mathbf{0}_{3\times12} \in \mathbb{R}^{3\times12}$ is a zero matrix. The measurement noise 
covariance matrix $\mathbf{R}_{k} = \mathbb{E}[\boldsymbol{v}\boldsymbol{v}^{T}] \in \mathbb{R}^{3\times3}$ 
represents the uncertainty in the GNSS position measurements. The Kalman update steps are 
then carried out as:
\begin{equation}\label{eqn:gain}
\mathbf{K}_{k} = \mathbf{P}^{-}_{f,k}\mathbf{H}^{T}_{k}\left(\mathbf{H}_{k}\mathbf{P}^{-}_{f,k}\mathbf{H}^{T}_{k} + \mathbf{R}_{k}\right)^{-1}
\end{equation}
\begin{equation}\label{eqn:posteriori}
\mathbf{P}^{+}_{f,k} = \left(\mathbf{I} - \mathbf{K}_{k}\mathbf{H}_{k}\right)\mathbf{P}^{-}_{f,k}
\end{equation}
\begin{equation}\label{eqn:state_update}
\delta\boldsymbol{x}^{+}_{f,k} = \delta\boldsymbol{x}^{-}_{f,k} + \mathbf{K}_{k}\left(\delta\boldsymbol{z}_{k} - \mathbf{H}_k\,\delta\boldsymbol{x}^{-}_{f,k}\right).
\end{equation}
Lastly, the updated error-state vector is reset to zero, ensuring that the nominal state 
is realigned with the updated state:
\begin{equation}\label{eqn:reset}
\delta\boldsymbol{x}^{-}_{f,k+1} = \boldsymbol{0}
\end{equation}
The detailed derivations of the error-state EKF for INS/GNSS fusion presented in this 
paper follow the formulations outlined in~\cite{groves2015principles}.
\subsection{Backward Information Filter}
\noindent
The backward filter operates in reverse time, starting at epoch $k = n$. The backward 
estimates must remain independent of the forward filter estimates, therefore, none of 
the information used in the forward filter is permitted to be used in the backward filter. 
As a consequence, the backward covariance matrix $\mathbf{P}_{b,k} \in \mathbb{R}^{15\times15}$, 
where the subscript $(\cdot)_{b}$ denotes backward run, is initialized at $k = n$ 
to represent infinite uncertainty, i.e., no prior information, as~\cite{sarkka2023bayesian}:
\begin{equation}\label{eqn:infoCov}
\mathbf{P}_{b,n}^{-} = \infty, \quad \mathbf{\mathcal{I}}_{b,n}^{-} = \left(\mathbf{P}_{b,n}^{-}\right)^{-1} = \mathbf{0}
\end{equation}
where $\mathbf{\mathcal{I}}_{b,k} \in \mathbb{R}^{15\times15}$ is the backward information 
matrix, defined as the inverse of the backward covariance matrix. This initialization 
describes a state of complete ignorance, ensuring independence between the forward and 
backward filters. The backward information vector is then introduced as:
\begin{equation}\label{eqn:info_vec}
\delta\boldsymbol{s}_{b,k} = \mathbf{\mathcal{I}}_{b,k}\,\delta\boldsymbol{x}_{b,k}
\end{equation}
where $\delta\boldsymbol{x}_{b,k} \in \mathbb{R}^{15\times1}$ is the backward error-state 
vector and $\delta\boldsymbol{s}_{b,k} \in \mathbb{R}^{15\times1}$ is the backward information 
vector. In accordance with \eqref{eqn:infoCov}, this is initialized as 
$\delta\boldsymbol{s}_{b,n} = \boldsymbol{0}$, enforcing independence between the forward 
and backward filters at the starting epoch.
\\ \noindent
The discrete error-state information filter is propagated backwards in time as follows:
\begin{equation}\label{eqn:info_pred_state}
\delta\boldsymbol{x}^{-}_{b,k-1} = \mathbf{\Phi}_{k-1}^{-1}\,\delta\boldsymbol{x}^{+}_{b,k}
\end{equation}
where $\mathbf{\Phi}_{k-1}^{-1}$ is well defined for all $k$, since the state transition 
matrix $\mathbf{\Phi}_{k-1}$ is derived from the matrix exponential and is therefore always 
invertible. The backward covariance is propagated as:
\begin{equation}\label{eqn:info_pred_back}
\mathbf{P}^{-}_{b,k-1} = \mathbf{\Phi}^{-1}_{k-1}\left(\mathbf{P}^{+}_{b,k} + \mathbf{Q}_{b,k}\right)\mathbf{\Phi}^{-T}_{k-1}
\end{equation}
where the backward discrete process noise covariance matrix $\mathbf{Q}_{b,k} \in \mathbb{R}^{15\times15}$ 
is approximated as:
\begin{equation}\label{eqn:Qb}
\mathbf{Q}_{b,k} = \frac{1}{2}\,\mathbf{\Phi}^{-1}_{k}\,\mathbf{G}_{k}\,\mathbf{Q}\,\mathbf{G}_{k}^{T}\,\mathbf{\Phi}^{-T}_{k}\,dt
\end{equation}
The update phase of the backward filter is carried out as follows:
\begin{equation}\label{eqn:Iplus}
\mathbf{\mathcal{I}}^{+}_{b,k} = \mathbf{\mathcal{I}}^{-}_{b,k} + \mathbf{H}^{T}_{k}\mathbf{R}^{-1}_{k}\mathbf{H}_{k}
\end{equation}
\begin{equation}\label{eqn:splus}
\delta\boldsymbol{s}^{+}_{b,k} = \mathbf{\mathcal{I}}^{-}_{b,k}\,\delta\boldsymbol{x}^{-}_{f,k} + \mathbf{H}^{T}_{k}\mathbf{R}^{-1}_{k}\delta\boldsymbol{z}_{k} + \mathbf{\mathcal{I}}^{-}_{b,k}\,\delta\boldsymbol{x}^{-}_{b,k}.
\end{equation}
A more detailed derivation of \eqref{eqn:infoCov}--\eqref{eqn:splus} can be found 
in~\cite{liu2010two}.
\subsection{Two-Filter Smoother}\label{TFS}
\noindent
The TFS combines the forward error-state EKF and a backward information filter to 
produce a smoothed state estimate at each epoch $k$. Suppose the forward and backward 
state estimates are combined into a smoothed estimate as:
\begin{equation}\label{eqn:tfs_full_fuse}
\boldsymbol{x}_{s,k} = \mathbf{K}_{f}\,\boldsymbol{x}_{f,k} + \mathbf{K}_{b}\,\boldsymbol{x}_{b,k}
\end{equation}
where $\mathbf{K}_{f} \in \mathbb{R}^{15\times15}$ and $\mathbf{K}_{b} \in \mathbb{R}^{15\times15}$ 
are the forward and backward fusion gain matrices, respectively. Since both estimates are 
unbiased, and to ensure the fused estimate $\boldsymbol{x}_{s,k}$ is also unbiased, the 
constraint $\mathbf{K}_{f} + \mathbf{K}_{b} = \mathbf{I}$ is enforced, giving:
\begin{equation}\label{eqn:tfs_full_fuse2}
\boldsymbol{x}_{s,k} = \mathbf{K}_{f}\,\boldsymbol{x}_{f,k} + \left(\mathbf{I} - \mathbf{K}_{f}\right)\boldsymbol{x}_{b,k}
\end{equation}
The covariance of the smoothed estimate can then be derived as follows:
\begin{equation}\label{eqn:tfs_cov_deriv}
\begin{aligned}
\mathbf{P}_{s,k} &= \mathbb{E}\left[\left(\boldsymbol{x}^{NOM}_{k} - \boldsymbol{x}_{s,k}\right)\left(\boldsymbol{x}^{NOM}_{k} - \boldsymbol{x}_{s,k}\right)^{T}\right] \\
&= \mathbb{E}\Big\{\mathbf{K}_{f}\left(\delta\boldsymbol{x}_{f,k}\delta\boldsymbol{x}_{f,k}^{T} + \delta\boldsymbol{x}_{b,k}\delta\boldsymbol{x}_{b,k}^{T}\right)\mathbf{K}_{f}^{T} \\
&\qquad + \delta\boldsymbol{x}_{b,k}\delta\boldsymbol{x}_{b,k}^{T} - \mathbf{K}_{f}\delta\boldsymbol{x}_{b,k}\delta\boldsymbol{x}_{b,k}^{T} \\
&\qquad - \delta\boldsymbol{x}_{b,k}\delta\boldsymbol{x}_{b,k}^{T}\mathbf{K}_{f}^{T}\Big\}
\end{aligned}
\end{equation}
where $\delta\boldsymbol{x}_{f,k} = \boldsymbol{x}^{NOM}_{k} - \boldsymbol{x}_{f,k}$ and 
$\delta\boldsymbol{x}_{b,k} = \boldsymbol{x}^{NOM}_{k} - \boldsymbol{x}_{b,k}$ are the forward 
and backward error-state vectors as defined in \eqref{eqn:error_state}, and where the 
cross-covariance term vanishes due to the independence of the forward and backward 
estimates, i.e., $\mathbb{E}\left[\delta\boldsymbol{x}_{f,k}\,\delta\boldsymbol{x}^{T}_{b,k}\right] = \mathbf{0}$. 
Defining $\mathbf{P}_{f,k} = \mathbb{E}[\delta\boldsymbol{x}_{f,k}\,\delta\boldsymbol{x}^{T}_{f,k}]$ 
and $\mathbf{P}_{b,k} = \mathbb{E}[\delta\boldsymbol{x}_{b,k}\,\delta\boldsymbol{x}^{T}_{b,k}]$ 
as the forward and backward error-state covariance matrices, respectively, a minimum 
variance estimate is obtained by taking the derivative of the trace of $\mathbf{P}_{s,k}$ 
with respect to $\mathbf{K}_{f}$, setting it to zero and
solving for the optimal $\mathbf{K}_{f}$:
\begin{equation}\label{eqn:tfs_Kf}
\mathbf{K}_{f} = \mathbf{P}_{b,k}\left(\mathbf{P}_{f,k} + \mathbf{P}_{b,k}\right)^{-1}
\end{equation}
and consequently:
\begin{equation}\label{eqn:tfs_Kb}
\mathbf{K}_{b} = \mathbf{I} - \mathbf{K}_{f} = \mathbf{P}_{f,k}\left(\mathbf{P}_{f,k} + \mathbf{P}_{b,k}\right)^{-1}
\end{equation}
Substituting \eqref{eqn:tfs_Kf} back into \eqref{eqn:tfs_cov_deriv}, after some algebra, 
the smoothed covariance simplifies to:
\begin{equation}\label{eqn:tfs_cov}
\mathbf{P}_{s,k} = \left(\mathbf{P}^{-1}_{f,k} + \mathbf{P}^{-1}_{b,k}\right)^{-1} = \left(\mathbf{\mathcal{I}}_{f,k} + \mathbf{\mathcal{I}}_{b,k}\right)^{-1}
\end{equation}
The result in \eqref{eqn:tfs_cov} shows that the information from both filters is 
additively combined, which is a direct consequence of the independence between the 
forward and backward estimates. Ultimately, the smoothed full state estimate in 
\eqref{eqn:tfs_full_fuse} becomes:
\begin{equation}\label{eqn:tfs_state_full}
\boldsymbol{x}_{s,k} = \mathbf{P}_{s,k}\left(\mathbf{\mathcal{I}}_{f,k}\,\boldsymbol{x}_{f,k} + \mathbf{\mathcal{I}}_{b,k}\,\boldsymbol{x}_{b,k}\right)
\end{equation}
\begin{proposition}\label{prop:tfs_equivalence}
TFS fusion can be performed on the error-state alone, yielding the same smoothed 
full state after correction of the nominal state:
\begin{equation}\label{eqn:tfs_error_fuse}
\delta\boldsymbol{x}_{s,k} = \mathbf{P}_{s,k}\left(\mathbf{\mathcal{I}}_{f,k}\,\delta\boldsymbol{x}_{f,k} + \mathbf{\mathcal{I}}_{b,k}\,\delta\boldsymbol{x}_{b,k}\right)
\end{equation}
\end{proposition}

\begin{proof}
Recall the error-state decomposition from \eqref{eqn:error_state}:
\begin{equation}\label{eqn:tfs_decomp}
\boldsymbol{x}_{f,k} = \boldsymbol{x}^{NOM}_{k} - \delta\boldsymbol{x}_{f,k}, \quad 
\boldsymbol{x}_{b,k} = \boldsymbol{x}^{NOM}_{k} - \delta\boldsymbol{x}_{b,k}
\end{equation}
Substituting \eqref{eqn:tfs_decomp} into \eqref{eqn:tfs_state_full}:
\begin{align}\label{eqn:tfs_sub1}
\boldsymbol{x}_{s,k} = \mathbf{P}_{s,k}\big(&\mathbf{\mathcal{I}}_{f,k}(\boldsymbol{x}^{NOM}_{k} - \delta\boldsymbol{x}_{f,k}) \nonumber\\
&+ \mathbf{\mathcal{I}}_{b,k}(\boldsymbol{x}^{NOM}_{k} - \delta\boldsymbol{x}_{b,k})\big)
\end{align}
Expanding and collecting terms in $\boldsymbol{x}^{NOM}_{k}$:
\begin{align}\label{eqn:tfs_sub2}
\boldsymbol{x}_{s,k} &= \mathbf{P}_{s,k}\left(\mathbf{\mathcal{I}}_{f,k} + \mathbf{\mathcal{I}}_{b,k}\right)\boldsymbol{x}^{NOM}_{k} \nonumber\\
&\quad - \mathbf{P}_{s,k}\left(\mathbf{\mathcal{I}}_{f,k}\,\delta\boldsymbol{x}_{f,k} + \mathbf{\mathcal{I}}_{b,k}\,\delta\boldsymbol{x}_{b,k}\right)
\end{align}
From \eqref{eqn:tfs_cov}, $\mathbf{P}_{s,k}\left(\mathbf{\mathcal{I}}_{f,k} + 
\mathbf{\mathcal{I}}_{b,k}\right) = \mathbf{I}$, so the first term cancels to 
$\boldsymbol{x}^{NOM}_{k}$ and \eqref{eqn:tfs_sub2} simplifies to:
\begin{align}\label{eqn:tfs_sub3}
\boldsymbol{x}_{s,k} = \boldsymbol{x}^{NOM}_{k} - \mathbf{P}_{s,k}\left(\mathbf{\mathcal{I}}_{f,k}\,\delta\boldsymbol{x}_{f,k} + \mathbf{\mathcal{I}}_{b,k}\,\delta\boldsymbol{x}_{b,k}\right)
\end{align}
Applying the error-state correction from \eqref{eqn:error_state} using 
\eqref{eqn:tfs_error_fuse}:
\begin{align}\label{eqn:tfs_proof_final}
\boldsymbol{x}_{s,k} &= \boldsymbol{x}^{NOM}_{k} - \delta\boldsymbol{x}_{s,k} \nonumber\\
&= \boldsymbol{x}^{NOM}_{k} - \mathbf{P}_{s,k}\left(\mathbf{\mathcal{I}}_{f,k}\,\delta\boldsymbol{x}_{f,k} + \mathbf{\mathcal{I}}_{b,k}\,\delta\boldsymbol{x}_{b,k}\right)
\end{align}
Since \eqref{eqn:tfs_sub3} and \eqref{eqn:tfs_proof_final} are identical, it is 
confirmed that operating on the error-state via \eqref{eqn:tfs_error_fuse} and 
correcting the nominal state is sufficient to recover the optimal smoothed full 
navigation state.
\end{proof}
\noindent
The TFS therefore provides an optimal smoothed estimate over the entire trajectory by 
additively combining the information from the forward and backward filters at each epoch, 
without requiring storage of the full forward filter history as in the RTSS, making it 
particularly well-suited for post-processing navigation applications. A full derivation and theory of the TFS can be seen, for example, in~\cite{simon2006optimal}.
\subsection{Rauch--Tung--Striebel Smoother}\label{RTSS}
\noindent
The RTSS is a fixed-interval backward smoother that 
operates on the stored results of the forward Kalman filter pass described in the preceding subsection. Since all required variables, namely the a \textit{priori} and a \textit{posteriori} error-state estimates $\delta\boldsymbol{x}^{-}_{f,k}$ and $\delta\boldsymbol{x}^{+}_{f,k}$, 
and their associated covariance matrices $\mathbf{P}^{-}_{f,k}$ and $\mathbf{P}^{+}_{f,k}$, 
are saved during the forward pass, the RTSS requires no additional propagation steps, 
making it computationally efficient for post-processing applications. The RTSS equations 
are given by:
\begin{equation}\label{eqn:Ks}
\mathbf{K}_{s,k} = \mathbf{P}^{+}_{f,k}\,\mathbf{\Phi}_{k}\left(\mathbf{P}^{-}_{f,k+1}\right)^{-1}
\end{equation}
\begin{equation}\label{eqn:Ps}
\mathbf{P}_{s,k} = \mathbf{P}^{+}_{f,k} + \mathbf{K}_{s,k}\left[\mathbf{P}_{s,k+1} - \mathbf{P}^{-}_{f,k+1}\right]\mathbf{K}^{T}_{s,k}
\end{equation}
\begin{equation}\label{eqn:dxs}
\delta\boldsymbol{x}_{s,k} = \delta\boldsymbol{x}^{+}_{f,k} + \mathbf{K}_{s,k}\left[\delta\boldsymbol{x}_{s,k+1} - \delta\boldsymbol{x}^{-}_{f,k+1}\right]
\end{equation}
\begin{equation}\label{eqn:xs}
\boldsymbol{x}_{s,k} = \boldsymbol{x}^{NOM}_{f,k} - \delta\boldsymbol{x}_{s,k}
\end{equation}
where $\mathbf{K}_{s,k} \in \mathbb{R}^{15\times15}$ is the smoother gain matrix. The smoother is initialized at the last epoch $k = n$ 
using the forward filter's final estimates, i.e., $\mathbf{P}_{s,n} = \mathbf{P}^{+}_{f,n}$ 
and $\delta\boldsymbol{x}_{s,n} = \delta\boldsymbol{x}^{+}_{f,n}$.
\section{Motivation}\label{sec:motivation}
\noindent
In INS/GNSS navigation, the fused solution typically achieves $2-5$ meters positional accuracy, whereas GNSS RTK provides centimeter-level accuracy. The accuracy discrepancy between the two systems manifests primarily as a systematic position offset rather than zero-mean noise~\cite{falco2017loose}. Since smoothing algorithms operate by reducing stochastic noise, they are fundamentally unable to correct such an offset, which motivates the problem addressed in this work.
\begin{figure}[!h]
    \centering
    \includegraphics[width=\columnwidth]{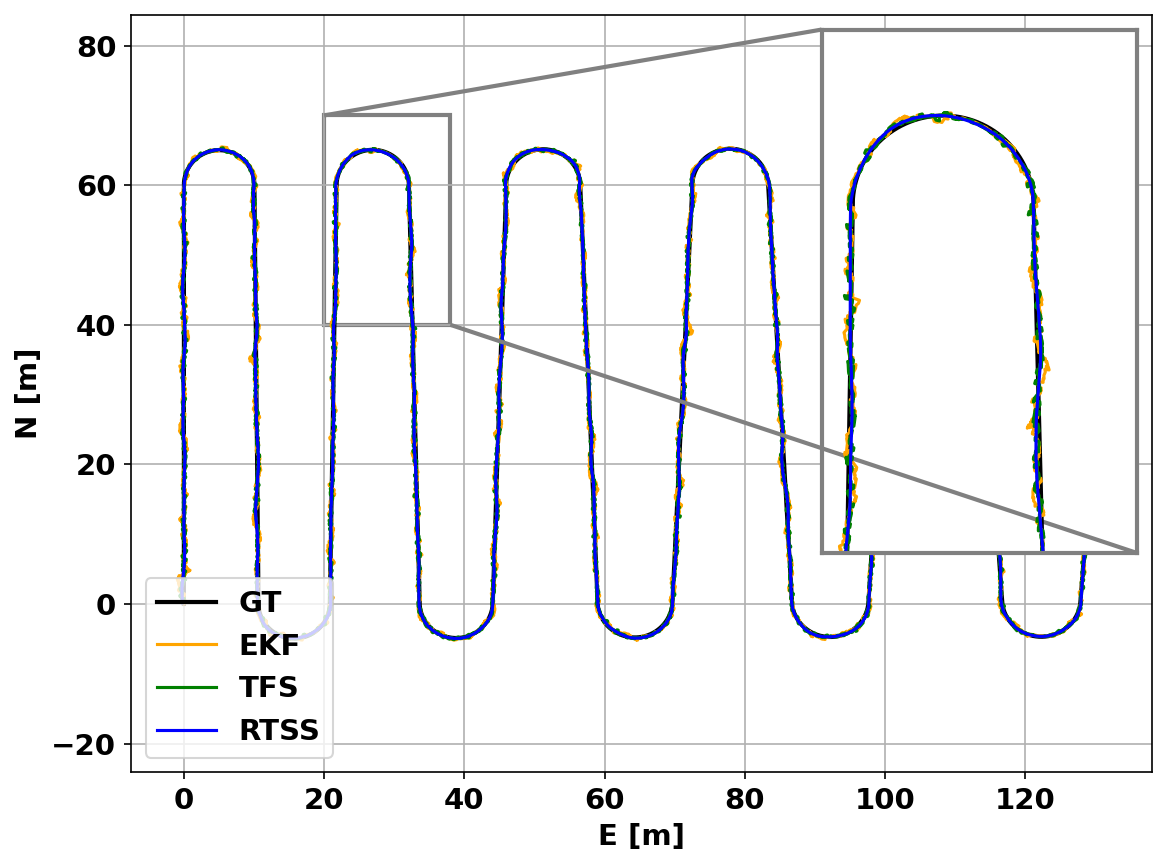}
    \caption{Estimated trajectory with zero-mean GNSS noise 
    ($\mu = 0$ [m], $\sigma = 0.5$ [m]). All estimators closely 
    follow the GT, confirming standard smoothing performance 
    under unbiased measurements.}
    \label{fig:traj_0}
\end{figure}
\noindent
To illustrate the problem, a simulation of an INS/GNSS loosely-coupled system is 
conducted over a 400-second lawnmower trajectory. The inertial data is corrupted 
with additive white noise equivalent to a consumer-grade micro-electro-mechanical system (MEMS) IMU, with standard 
deviations of $0.0316$ [rad/s] and $32.2$ [mg] for the gyroscope and accelerometer, 
respectively. The INS, operating in 100\,Hz, is then updated by GNSS position measurements, in 10\,Hz. We examine three different noise configurations: 
\begin{enumerate}
    \item Zero-mean Gaussian noise with a standard deviation of $0.5$ [m].
    \item Gaussian noise with a mean of $1.5$ [m] and a standard deviation of $0.5$ [m].
    \item Gaussian noise with a mean of $3$ [m] and a standard deviation of $0.5$ [m].
\end{enumerate}
\noindent
This allows the effect of a systematic position bias on the smoothed 
solution to be isolated and clearly demonstrated.
\\ \noindent
Figs.~\ref{fig:traj_0}, \ref{fig:traj_1_5}, and \ref{fig:traj_3} present the 
estimated trajectories, using EKF (Section~\ref{EKF}), TFS (Section~\ref{TFS}) and RTSS (Section~\ref{RTSS}), for the three noise configurations. In the zero-mean case 
(Fig.~\ref{fig:traj_0}), all three estimators, the EKF, TFS, and RTSS, closely 
follow the ground truth (GT) trajectory, confirming that standard smoothing effectively 
reduces noise when the GNSS measurements are unbiased. However, as a systematic 
bias is introduced, the behavior of the estimators shifts accordingly. In the $1.5$ [m] 
bias case (Fig.~\ref{fig:traj_1_5}), the EKF solution is visibly offset from the 
GT, while the RTSS and TFS smoothed solutions, despite reducing noise, inherit the same bias and remain offset. This effect is further 
amplified in the $3$ [m] bias case (Fig.~\ref{fig:traj_3}), where the offset 
between all estimators and the GT is clearly visible even in the 
full-scale trajectory. The enlarged insets confirm that smoothing improves the precision of the solution, reducing the spread around the mean. However, it does not correct the systematic offset, as the smoothed trajectories remain displaced from the GT by the same bias magnitude as the EKF.
\begin{figure}[!h]
    \centering
    \includegraphics[width=\columnwidth]{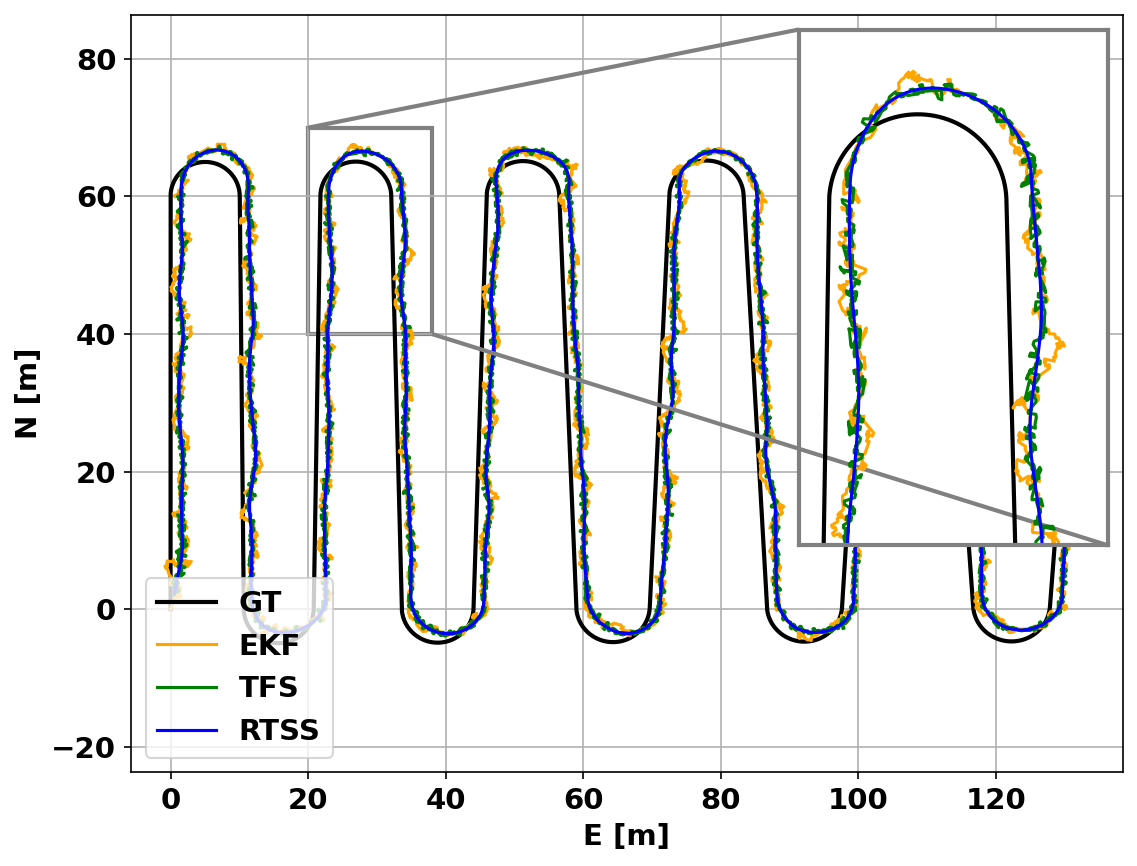}
    \caption{Estimated trajectory with biased GNSS noise 
    ($\mu = 1.5$ [m], $\sigma = 0.5$ [m]). The EKF, RTSS, and TFS 
    solutions all exhibit a systematic offset from the GT, 
    demonstrating that smoothing cannot correct a position bias.}
    \label{fig:traj_1_5}
\end{figure}
\begin{figure}[!h]
    \centering
    \includegraphics[width=\columnwidth]{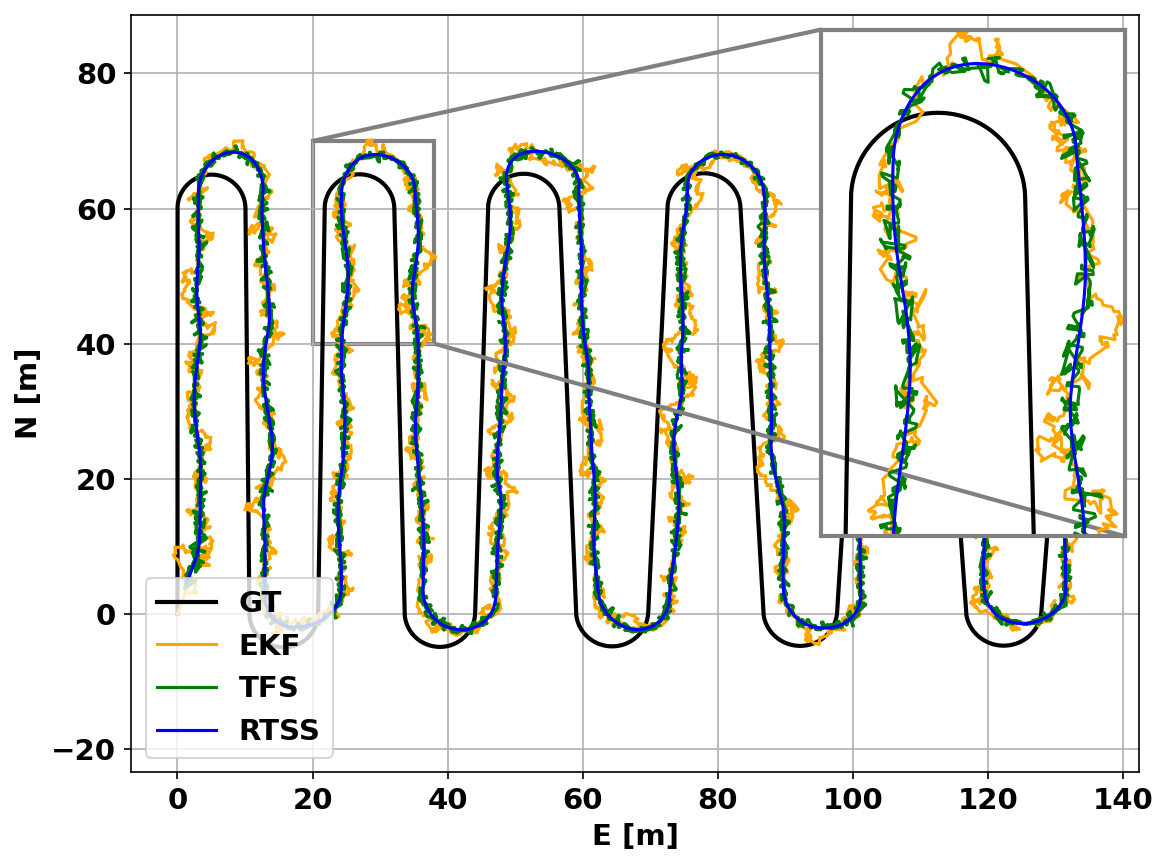}
    \caption{Estimated trajectory with biased GNSS noise 
    ($\mu = 3$ [m], $\sigma = 0.5$ [m]). The systematic offset 
    is clearly amplified relative to Fig.~\ref{fig:traj_1_5}, 
    further confirming that the bias is inherited by the smoothed 
    solutions regardless of the smoothing algorithm used.}
    \label{fig:traj_3}
\end{figure}

\noindent
The above demonstrates that conventional smoothing algorithms are fundamentally limited in the presence of biased GNSS measurements, as they inherit the systematic offset from the filter and are unable to correct it in the backward pass. Ultimately, this implies that the smoothing framework lacks sufficient information to bridge the gap between ordinary GNSS positioning accuracy and RTK-level precision. This raises the central question addressed in this paper: can neural-aided approaches be incorporated into the smoothing mechanism to improve navigation accuracy while remaining consistent with the theoretical foundations of Bayesian estimation?
\section{Proposed Approach}\label{sec:proposed}
\noindent
To improve navigation accuracy through a software-only implementation, we propose a neural-aided approach termed BLENDS, Bayesian learning-enhanced navigation with deep smoothing. The proposed method incorporates a deep learning component directly into the TFS framework, where modification matrices are learned to adjust the forward and backward estimated covariances. The modified covariances, together with an additive correction term, are then used to compute the smoothed error-state and its associated covariance. Consequently, BLENDS has the potential to correct the systematic bias that conventional smoothing algorithms are unable to address.
\subsection{BLENDS Framework}
\noindent
The BLENDS framework augments the classical TFS by introducing two learned quantities per time step: covariance modification matrices $\mathbf{D}_{f,k}$ and $\mathbf{D}_{b,k}$, and an additive correction vector $\boldsymbol{c}_k$, produced by a neural network $\mathcal{F}_{\boldsymbol{\theta}}$. Given these outputs, the forward and backward TFS covariance matrices are computed as:
\begin{equation}
    \tilde{\mathbf{P}}_{f,k} = \mathbf{D}_{f,k}\,\mathbf{P}_{f,k}\,
    \mathbf{D}_{f,k}^T, \quad
    \tilde{\mathbf{P}}_{b,k} = \mathbf{D}_{b,k}\,\mathbf{P}_{b,k}\,
    \mathbf{D}_{b,k}^T.
    \label{eqn:Pscaled}
\end{equation}
Although modified, the forward and backward covariance matrices~\eqref{eqn:Pscaled} preserve their essential positive semi-definite property. Since 
$\mathbf{P}^+_{f,k} \succ 0$ by construction of the Kalman filter and 
$\mathbf{D}_{f,k}$ is full-rank, the congruence transformation property 
guarantees that $\tilde{\mathbf{P}}_{f,k} \succ 0$, and equivalently 
for $\tilde{\mathbf{P}}_{b,k}$. This ensures that the modified information 
matrices $\tilde{\mathbf{\mathcal{I}}}_{f,k} = \tilde{\mathbf{P}}^{-1}_{f,k}$ 
and $\tilde{\mathbf{\mathcal{I}}}_{b,k} = \tilde{\mathbf{P}}^{-1}_{b,k}$ 
remain well-defined throughout the entire trajectory. The modified covariance matrices~\eqref{eqn:Pscaled} are substituted into~\eqref{eqn:tfs_Kf}--\eqref{eqn:tfs_Kb} to compute the fusion gains, producing hybrid learned gains rather than the standard model-based gains of the classical TFS.
\\ \noindent
The additive correction term $\boldsymbol{c}_k$ provides the network 
with a direct mechanism to compensate for systematic biases that the 
covariance modification alone cannot address. It is bounded element-wise via 
a scaled hyperbolic tangent:
\begin{equation}
    \boldsymbol{c}_k = \tanh\!\left(\hat{\boldsymbol{c}}_k\right) 
    \odot \boldsymbol{m}_k
    \label{eqn:cor}
\end{equation}
where $\odot$ denotes element-wise multiplication and 
$\boldsymbol{m}_k \in \mathbb{R}^{15}$ is a time-varying bound vector 
that contracts during training via a ramp $\rho \in [0,1]$:
\begin{equation}
    \boldsymbol{m}_k = (1 - \rho)\,\boldsymbol{m}^{\text{wide}} + 
    \rho\,\boldsymbol{m}^{\text{base}},
    \label{eqn:ramp}
\end{equation}
$\boldsymbol{m}^{\text{wide}}$ and $\boldsymbol{m}^{\text{base}}$ 
define the initial and final per-state correction bounds, respectively, the ramp $\rho$ is defined as:
\begin{equation}
    \rho(e) = \left(\min\!\left(\max\!\left(\frac{e}{e_w}, 
    0\right), 1\right)\right)^p,
    \label{eqn:ramp_fn}
\end{equation}
$e$ is the current training epoch, $e_w$ is the warm-up 
duration, and $p$ controls the rate of contraction. At the beginning 
of training ($e \ll e_w$), $\rho \approx 0$ and the network is 
permitted to produce large corrections within 
$\boldsymbol{m}^{\text{wide}}$, allowing broad exploration of the 
solution space. As training progresses ($e \rightarrow e_w$), 
$\rho \rightarrow 1$ and the bounds contract to 
$\boldsymbol{m}^{\text{base}}$, encouraging the network to produce 
small, precise corrections. This strategy serves as a physically 
motivated constraint that prevents the network from producing 
corrections that violate the expected dynamic ranges of the navigation 
states, and ensures that the learned outputs remain physically 
meaningful throughout training.
\\ \noindent
The final BLENDS smoothed error-state is obtained by substituting the 
modified covariances into the TFS fusion~\eqref{eqn:tfs_error_fuse} and 
appending the learned correction:
\begin{equation}
    \delta\boldsymbol{x}^{\text{BLENDS}}_{s,k} = 
    \tilde{\mathbf{P}}_{s,k}\!\left(
    \tilde{\mathbf{\mathcal{I}}}_{f,k}\,\delta\boldsymbol{x}_{f,k} + 
    \delta\boldsymbol{s}_{b,k}\right) + \boldsymbol{c}_k
    \label{eqn:blends_output}
\end{equation}
The BLENDS smoothed covariance is defined as the second central moment 
of the smoothed error-state:
\begin{equation}
    \tilde{\mathbf{P}}_{s,k} = \mathbb{E}\!\left[
    \delta\boldsymbol{x}^{\text{BLENDS}}_{s,k}\,
    (\delta\boldsymbol{x}^{\text{BLENDS}}_{s,k})^T
    \right]
    \label{eqn:blends_cov_def}.
\end{equation}
Substituting~\eqref{eqn:blends_output} 
into~\eqref{eqn:blends_cov_def} and using the independence of the 
forward and backward filter errors, the BLENDS smoothed covariance 
expands as:
\begin{equation}
    \tilde{\mathbf{P}}_{s,k} = 
    \tilde{\mathbf{K}}_{f,k}\,\tilde{\mathbf{P}}_{f,k}\,
    \tilde{\mathbf{K}}_{f,k}^T + 
    \tilde{\mathbf{K}}_{b,k}\,\tilde{\mathbf{P}}_{b,k}\,
    \tilde{\mathbf{K}}_{b,k}^T + 
    \boldsymbol{c}_k\boldsymbol{c}_k^T
    \label{eqn:blends_cov}
\end{equation}
where the first two terms represent the propagated uncertainty from 
the forward and backward filters, respectively, and the outer product 
$\boldsymbol{c}_k\boldsymbol{c}_k^T$ is a deterministic matrix that 
captures the additional uncertainty introduced by the additive 
correction term. The complete BLENDS inference procedure is 
summarized in Algorithm~\ref{alg:blends}.

\begin{algorithm}[!h]
\caption{BLENDS: Inference Pipeline}
\label{alg:blends}
\begin{algorithmic}[1]
\Require IMU data $\{\boldsymbol{f}_k, \boldsymbol{\omega}_k\}_{k=1}^{N}$, 
         GNSS positions $\{\boldsymbol{z}_k\}_{k=1}^{N}$,
         trained network $\mathcal{F}_{\boldsymbol{\theta}}$
\Ensure  Smoothed navigation states 
         $\{\boldsymbol{x}^{\text{BLENDS}}_{s,k}\}_{k=1}^{N}$
\Statex \textbf{--- Forward Pass ---}
\For{$k = 1$ \textbf{to} $N$}
    \State Propagate INS mechanization equations
    \State Predict error-state and covariance via~\eqref{eqn:pred_state},~\eqref{eqn:pred}
    \If{GNSS measurement available}
        \State Update $\delta\boldsymbol{x}^+_{f,k}$, $\mathbf{P}^+_{f,k}$ 
               via~\eqref{eqn:gain}--\eqref{eqn:reset}
    \EndIf
    \State Store $\delta\boldsymbol{x}_{f,k}$, $\mathbf{P}_{f,k}$, 
           $\boldsymbol{x}^{\text{NOM}}_{f,k}$, $\mathbf{\Phi}_k$
\EndFor
\Statex \textbf{--- Backward Pass ---}
\State Initialize $\mathbf{\mathcal{I}}^-_{b,N} = \mathbf{0}$, 
       $\delta\boldsymbol{s}^+_{b,N} = \boldsymbol{0}$
\For{$k = N$ \textbf{down to} $1$}
    \State Update $\mathbf{\mathcal{I}}^+_{b,k}$, $\delta\boldsymbol{s}^+_{b,k}$ 
           via~\eqref{eqn:Iplus},~\eqref{eqn:splus}
    \State Propagate $\mathbf{\mathcal{I}}^-_{b,k-1}$, $\delta\boldsymbol{x}^-_{b,k-1}$ 
           via~\eqref{eqn:info_pred_back},~\eqref{eqn:info_pred_state}
    \State Store $\delta\boldsymbol{x}_{b,k}$, $\mathbf{P}_{b,k}$
\EndFor
\Statex \textbf{--- BLENDS Fusion ---}
\State Assemble input sequence $\{\boldsymbol{u}_k\}_{k=1}^{N}$ 
       via~\eqref{eqn:net_input}
\State Compute $\{\mathbf{D}_{f,k}, \mathbf{D}_{b,k}, \boldsymbol{c}_k\}_{k=1}^{N} 
       = \mathcal{F}_{\boldsymbol{\theta}}\!\left(\boldsymbol{u}_{1:N}\right)$
       via~\eqref{eqn:net_output}
\For{$k = 1$ \textbf{to} $N$}
    \State Scale covariances $\tilde{\mathbf{P}}_{f,k}$, $\tilde{\mathbf{P}}_{b,k}$ 
           via~\eqref{eqn:Pscaled}
    \State Compute fusion gains $\tilde{\mathbf{K}}_{f,k}$, $\tilde{\mathbf{K}}_{b,k}$ 
           via~\eqref{eqn:tfs_Kf},~\eqref{eqn:tfs_Kb}
    \State Compute smoothed error-state 
           $\delta\boldsymbol{x}^{\text{BLENDS}}_{s,k}$ via~\eqref{eqn:blends_output}
    \State Compute smoothed covariance 
           $\tilde{\mathbf{P}}_{s,k}$ via~\eqref{eqn:blends_cov}
    \State Recover full state 
           $\boldsymbol{x}^{\text{BLENDS}}_{s,k} = \boldsymbol{x}^{\text{NOM}}_{f,k} 
           - \delta\boldsymbol{x}^{\text{BLENDS}}_{s,k}$
\EndFor
\State \Return $\{\boldsymbol{x}^{\text{BLENDS}}_{s,k},\, 
       \tilde{\mathbf{P}}_{s,k}\}_{k=1}^{N}$
\end{algorithmic}
\end{algorithm}
\begin{figure*}[h!]
    \centering
    \includegraphics[width=\textwidth]{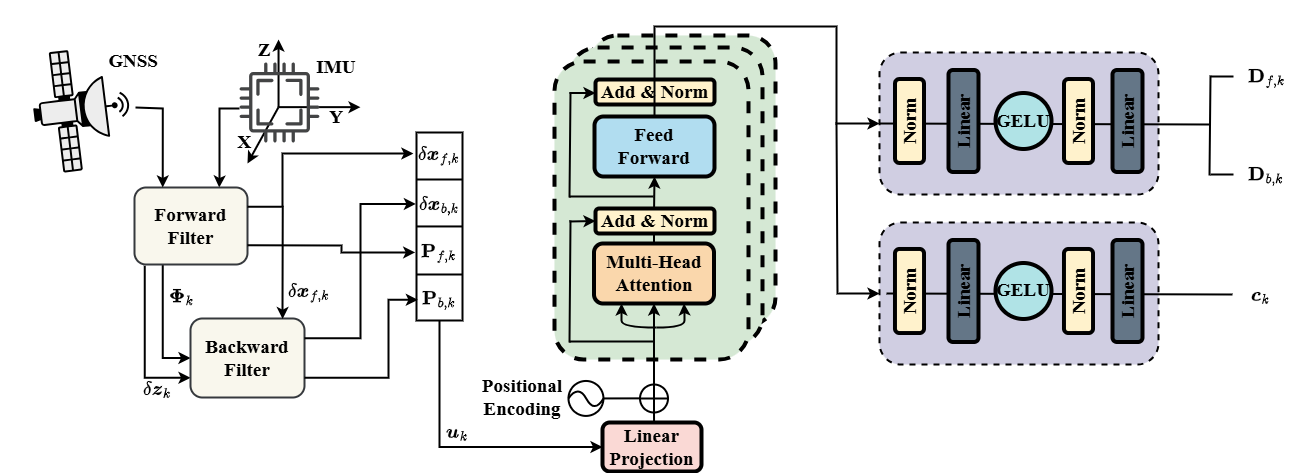}
    \caption{Overview of the BLENDS architecture. The forward EKF and 
    backward information filter outputs are concatenated and passed 
    through a transformer encoder, which produces covariance modification 
    matrices $\mathbf{D}_{f,k}$, $\mathbf{D}_{b,k}$ and an additive 
    correction vector $\boldsymbol{c}_k$. These are used to augment 
    the Two-Filter Smoother fusion, yielding the final BLENDS smoothed 
    state $\delta\boldsymbol{x}^{\text{BLENDS}}_{s,k}$.}
    \label{fig:architecture}
\end{figure*}
\subsection{Network Architecture}
\noindent
The input to the network is formed by concatenating a time window of 
the forward and backward error-state vectors with their associated 
flattened covariance matrices:
\begin{equation}
    \boldsymbol{u}_k = \left[\delta\boldsymbol{x}_{f,k},\; 
    \delta\boldsymbol{x}_{b,k},\; 
    \text{vec}(\mathbf{P}_{f,k}),\; 
    \text{vec}(\mathbf{P}_{b,k})\right] \in \mathbb{R}^{480}
    \label{eqn:net_input}
\end{equation}
where $\text{vec}(\cdot)$ denotes the vectorization of a matrix. The 
full sequence $\{\boldsymbol{u}_k\}_{k=1}^{N}$ is processed jointly 
by a transformer encoder, which exploits temporal dependencies across 
the trajectory to produce three outputs per time step:
\begin{equation}
    (\mathbf{D}_{f,k},\; \mathbf{D}_{b,k},\; \boldsymbol{c}_k )= 
    \mathcal{F}_{\boldsymbol{\theta}}\!\left(\boldsymbol{u}_{1:N}\right)
    \label{eqn:net_output}
\end{equation}
where $\mathbf{D}_{f,k}, \mathbf{D}_{b,k} \in \mathbb{R}^{15\times15}$ 
are covariance modification matrices, $\boldsymbol{c}_k \in \mathbb{R}^{15}$ 
is an additive correction vector, and $\boldsymbol{\theta}$ denotes the 
trainable network parameters. The BLENDS architecture is illustrated 
in Fig.~\ref{fig:architecture}. The input sequence is projected to 
embedding dimension $d_{\text{model}}$ via a linear layer, followed by 
sinusoidal positional encoding to preserve temporal order. The embedded 
sequence is processed by a transformer encoder~\cite{vaswani2017attention} 
of $L$ layers, each with $H$ attention heads, feed-forward 
hidden dimension $d_{\text{ff}}$, and GELU 
activations~\cite{hendrycks2016gaussian}, enabling the network to 
exploit long-range temporal dependencies across the trajectory.
\\ \noindent
The encoder output is passed to two independent prediction heads, each 
consisting of two linear layers with a GELU activation and layer 
normalization, with hidden dimension $d_{\text{head}}$. The covariance 
modification head outputs $\hat{\mathbf{D}}_{f,k}$ and $\hat{\mathbf{D}}_{b,k}$, 
and the correction head outputs $\hat{\boldsymbol{c}}_k$. The modification 
matrices are parameterized as near-identity residual perturbations:
\begin{equation}
    \mathbf{D}_{f,k} = \mathbf{I}_{15} + \alpha\,
    \tanh\!\left(\hat{\mathbf{D}}_{f,k}\right), \quad
    \mathbf{D}_{b,k} = \mathbf{I}_{15} + \alpha\,
    \tanh\!\left(\hat{\mathbf{D}}_{b,k}\right)
    \label{eqn:Dscale}
\end{equation}
where $\alpha \ll 1$ ensures that the modification matrices remain close to 
the identity, limiting the network's influence on the filter covariances 
to a small learned perturbation. Both heads are initialized with zero 
weights and biases, ensuring that at the start of training BLENDS 
reduces exactly to the standard TFS, providing a stable and physically 
meaningful initialization point.
\subsection{Bayesian-Consistent Loss Function}
\noindent
A key design principle of BLENDS is that the network output must remain 
consistent with the theoretical foundations of Bayesian estimation. 
Specifically, the smoothed state should represent the minimum-variance 
unbiased estimate, and the smoothed covariance should reflect 
the true uncertainty of the solution. To enforce these properties, we propose a Bayesian-consistent loss (BCL) that jointly supervises the mean and covariance of the smoothed output, comprising four loss terms: 1) position loss $\mathcal{L}^{(p)}_k$, 2) velocity loss $\mathcal{L}^{(v)}_k$, 3) rotation loss $\mathcal{L}^{(R)}_k$, and 4) covariance loss $\mathcal{L}^{(c)}_k$. The BCL is defined as:
\begin{equation}
    \mathcal{L}_{\text{BCL}} = \frac{1}{BT}\sum_{b=1}^{B}\sum_{k=1}^{T}
    \left(
    \lambda_p\,\mathcal{L}^{(p)}_k + 
    \lambda_v\,\mathcal{L}^{(v)}_k + 
    \lambda_r\,\mathcal{L}^{(R)}_k +
    \lambda_c\,\mathcal{L}^{(c)}_k
    \right)
    \label{eqn:total_loss}
\end{equation}
where $\lambda_p$, $\lambda_v$, $\lambda_r$, $\lambda_c$ are the 
weighting coefficients while $T$ and $B$ are the time window and batch size, respectively. The first three terms supervise the estimated 
mean, while the last term enforces covariance consistency.
\\ \noindent
The position loss $\mathcal{L}^{(p)}_k$ is computed in the local NED 
frame by converting the predicted and GT latitude, longitude, and altitude coordinates to NED displacements relative to a common reference point, taken as the first GT position in the batch. This normalization decouples the loss from absolute geographic coordinates, allowing the network to generalize across trajectories collected at arbitrary locations on the Earth without requiring position-specific training data. The position loss is defined as:
\begin{equation}
    \mathcal{L}^{(p)}_k = \ell_{\beta}\!\left(
    \boldsymbol{p}^{\text{NED},\text{GT}}_k - 
    \boldsymbol{p}^{\text{NED}}_k,\; 
    \boldsymbol{0}\right)
    \label{eqn:pos_loss}
\end{equation}
where $\ell_{\beta}(\cdot)$ denotes the Huber loss~\cite{huber1992robust} with threshold 
$\beta = 5.0$~m, which provides robustness to large position errors 
during early training.
\\ \noindent
The velocity loss is defined analogously:
\begin{equation}
    \mathcal{L}^{(v)}_k = \ell_{\beta}\!\left(
    \boldsymbol{v}^{\text{GT}}_k - \boldsymbol{v}_k,\; 
    \boldsymbol{0}\right).
    \label{eqn:vel_loss}
\end{equation}
The rotation loss $\mathcal{L}^{(R)}_k$ supervises the estimated 
attitude by penalizing the Frobenius norm of the difference between 
the predicted and GT rotation matrices:
\begin{equation}
    \mathcal{L}^{(R)}_k = \left\|
    \mathbf{R}^{\text{GT}}_k - \mathbf{R}_k
    \right\|^2_F
    \label{eqn:rot_loss}
\end{equation}
where $\|\cdot\|_F$ denotes the Frobenius norm. This formulation 
directly penalizes errors in all nine elements of the rotation matrix, 
providing a geometry-aware attitude supervision signal.
\\ \noindent
The covariance consistency term $\mathcal{L}^{(c)}_k$ minimizes the trace of the smoothed covariance matrix, and by doing so, enforces the minimum-variance property that is the theoretical optimality criterion of the Kalman smoother:
\begin{equation}
    \mathcal{L}^{(c)}_k = \text{Tr}\!\left(\mathbf{P}_{s,k}\right)
    \label{eqn:cov_loss}
\end{equation}
where the smoothed covariance is expressed in~\eqref{eqn:blends_cov}.
\\ \noindent
Jointly minimizing $\mathcal{L}^{(p)}_k$, $\mathcal{L}^{(v)}_k$, and 
$\mathcal{L}^{(R)}_k$ drives the smoothed mean towards the GT, while minimizing $\mathcal{L}^{(c)}_k$ ensures that the 
covariance matrix remains tight and physically meaningful, preventing 
the network from producing overconfident or inconsistent uncertainty 
estimates. Together, these terms ensure that BLENDS produces a 
Bayesian-consistent output in both mean and covariance.
\section{Analysis and Results}\label{sec:results}

\begin{figure}[!h]
    \centering
    \includegraphics[width=\columnwidth]{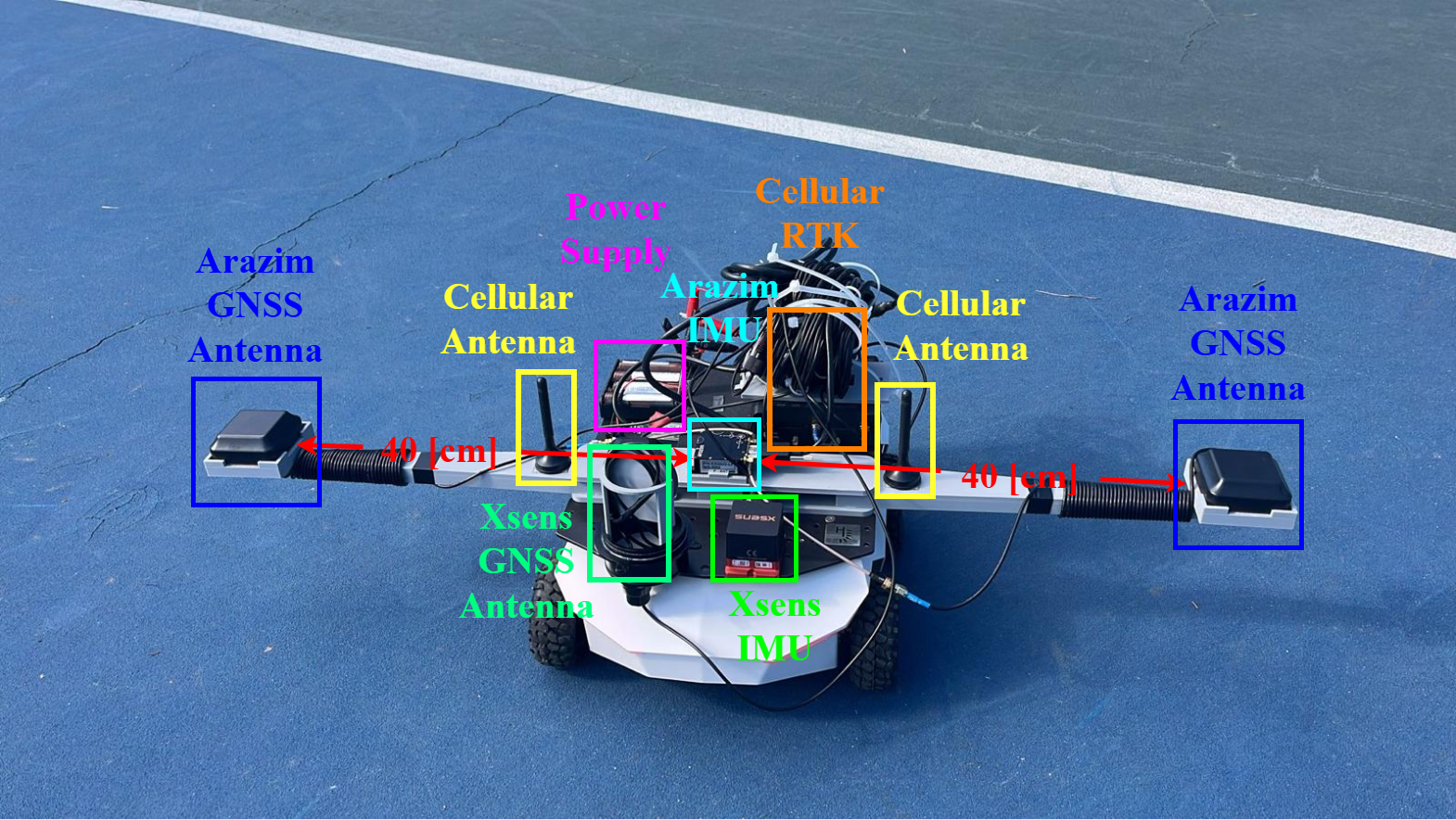}
    \caption{Experimental platform used for data collection. The mobile robot is equipped with two sensor suites: the 
        Arazim system, comprising a MEMS IMU and two 
        GNSS antennas mounted 40\,cm from the vehicle 
        center on a rigid crossbar, and the Xsens 
        system, comprising a consumer-grade IMU and a 
        single GNSS antenna mounted at the vehicle 
        body. A cellular RTK receiver provides 
        centimeter-level reference trajectories via a network RTK 
        correction service, supported by two cellular 
        antennas and a power supply.}
    \label{fig:arazim}
\end{figure}
\subsection{Datasets Description}
\subsubsection{Mobile Robot}
\noindent
Field data were collected using a Husarion ROSbot XL 
\cite{husarion_rosbotxl} mobile robot platform, equipped with two 
independent sensor suites operating simultaneously. The experimental 
configuration is illustrated in Fig.~\ref{fig:arazim}. The unit under test is the Arazim 
Exiguo\textsuperscript{\textregistered} EX-300
IMU~\cite{arazim_ex300} , operating at 100\,Hz. GNSS position updates are provided by the Xsens MTi-680G at 5\,Hz~\cite{xsens_mti680g}, serving as the measurement input to 
the navigation filter. The GT reference is the RTK solution integrated to 
the Arazim Exiguo\textsuperscript{\textregistered} EX-300 IMU at 
10\,Hz, providing centimeter-level position accuracy via a cellular 
network RTK correction service. In addition, the dual-antenna 
configuration with a baseline separation of 80\,cm provides 
sub-degree heading measurements. The RTK and the sub-degree heading measurements solutions does not contribute to the 
filter inputs and is used exclusively for performance evaluation.
\\ \noindent
The dataset comprises approximately 55 minutes of recorded data, 
divided into 30 trajectory segments spanning six motion patterns: 
square, infinity, sine, lawnmower, circle, and zigzag. The 
trajectories vary in duration and complexity, closed geometric 
paths, such as squares and circles, range from roughly 30 to 140 
seconds per lap, while open patterns, such as lawnmower and sine, 
sweeps extend up to 250 seconds per segment. Together they cover 
a broad range of heading changes, curvatures, and linear 
accelerations representative of mobile vehicle operation. The dataset is split such that the square, infinity, sine, lawnmower, circle, and zigzag segments are used for training, while the three mixed-maneuver trajectories are held out for validation and testing. 
The training set comprises approximately 49 minutes of data. The validation set, denoted MR-Val, consists of Trajectory~1 and 
comprises around 2 minutes, while the test set, denoted MR-Tes, 
includes Trajectory~2 and Trajectory~3, amounting to approximately 
4 minutes in total. Table~\ref{tab:imu_specs} summarizes the IMU specifications.
\begin{table}[!t]
    \centering
    \caption{Arazim EX-300 IMU Specifications.}
    \label{tab:imu_specs}
    \begin{tabular}{lc}
        \hline
        \textbf{Parameter} & \textbf{Arazim EX-300} \\
        \hline
        \multicolumn{2}{l}{\textit{Gyroscope}} \\
        In-run bias stability   & 2\,°/hr \\
        Angle random walk                     & 0.1\,°/$\sqrt{\text{hr}}$ \\
        \hline
        \multicolumn{2}{l}{\textit{Accelerometer}} \\
        In-run bias stability   & 20\,$\mu$g \\
        Velocity random walk                      & 0.05\,m/s/$\sqrt{\text{hr}}$ \\
        \hline
    \end{tabular}
\end{table}

\subsubsection{Quadrotor}
\noindent
The second dataset is the INSANE benchmark~\cite{brommer2022insane}, 
which employs an aerial quadrotor platform equipped with a Pixhawk4 
autopilot. An external RTK receiver provides centimeter-level GT, alongside additional sensors that are not utilized in this 
work. Specifically, this paper considers the Mars analogue desert 
subset, collected at the Ramon Crater in the Negev Desert, Israel. 
The Pixhawk4 IMU operates at 200\,Hz, GNSS position updates are 
provided at 5\,Hz, and the RTK GT at 8\,Hz. The onboard 
ICM-20689 IMU is a consumer-grade MEMS sensor, lower in both cost 
and performance compared to the Arazim EX-300 used 
in the mobile robot dataset, as reflected by its higher noise density 
and bias instability characteristics. The Mars analogue subset consists of 19 recorded trajectories, numbered 
trajectory~1 through trajectory~19. Following inspection of all trajectories for data 
quality and completeness, 13 trajectories were selected for use in this 
work. The selected trajectories are drawn primarily from zones 1 and 2, 
which are used for training and validation, while one trajectory from 
the entirely unseen zone 3 is held out for testing. The training set 
comprises approximately 23.7 minutes of data. The validation set, 
denoted QR-Val, consists of Trajectory~2 and comprises around 
2.8 minutes, while the test set, denoted QR-Tes, includes 
Trajectory~12 and Trajectory~15, amounting to approximately 
4.8 minutes of flight data in total. This split ensures that the test trajectory is representative of a genuinely unseen environment, providing a rigorous evaluation of the generalization capability of BLENDS. Full details regarding platform configuration, sensor calibration, and data collection procedures are available in the original publication~\cite{brommer2022insane}.
\subsection{Evaluation Metrics}
\noindent
The performance of each estimator is evaluated using the root mean square error (RMSE), defined as:
\begin{equation}
    \text{RMSE} = \sqrt{\frac{1}{T}\sum_{k=1}^{T}
    \left\|\boldsymbol{e}_k\right\|^2}
    \label{eqn:rmse}
\end{equation}
where $\boldsymbol{e}_k$ denotes the error vector at time step $k$. 
This metric is applied separately to position [m], velocity [m/s], 
and orientation [rad]. Additionally, to quantify the improvement in 
estimation uncertainty, the percent covariance improvement (PCI) is 
defined as the relative reduction in the trace of the smoothed 
covariance with respect to the forward EKF covariance:
\begin{equation}
    \text{PCI}_{j} = 100\times\frac{\text{Tr}\!\left(
    \mathbf{P}_{j}^{\text{ref}} - \mathbf{P}_{j}^{\text{test}} \right)}
    {\text{Tr}\!\left(\mathbf{P}_{j}^{\text{ref}}\right)}
    \label{eqn:pci}
\end{equation}
where $\mathbf{P}_{j}^{\text{ref}}$ is the forward EKF covariance and 
$\mathbf{P}_{j}^{\text{test}}$ is the smoothed covariance at the 
corresponding time step. A higher PCI value indicates a greater 
reduction in uncertainty achieved by the smoothing algorithm.

\subsection{Implementation Details}
\noindent
The deep neural network is implemented in PyTorch. All weights and 
biases of the final layers, those responsible for outputting the 
covariance modification matrices and the additive correction term, are 
initialized to zero, ensuring that the network reduces exactly to the 
standard TFS at initialization, as described in 
Section~\ref{sec:proposed}.

\begin{figure*}[!h]
    \centering
    \begin{subfigure}{0.32\textwidth}
        \includegraphics[width=\textwidth]{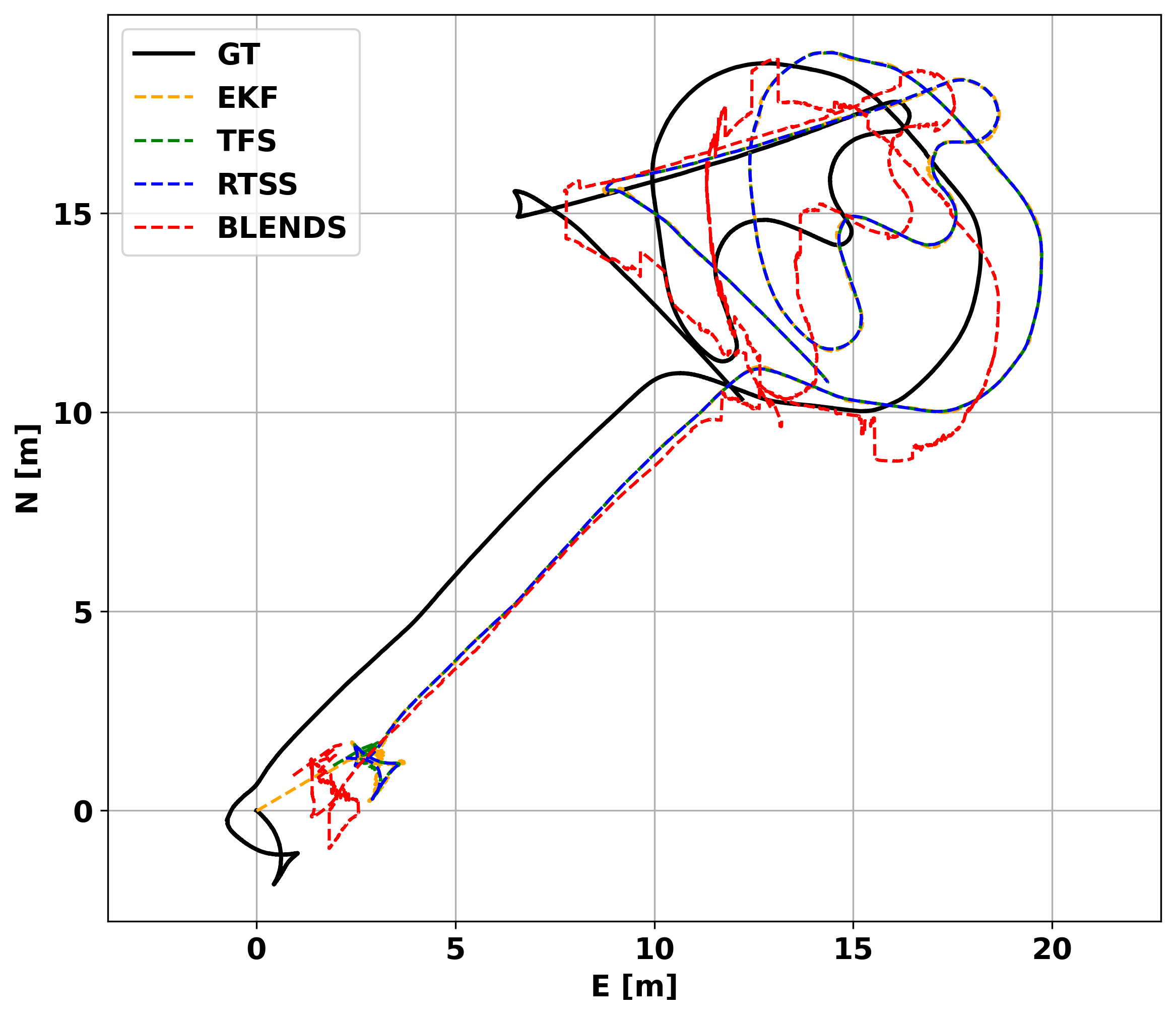}
        \caption{MR-Val (Traj.~1).}
    \end{subfigure}\hfill
    \begin{subfigure}{0.32\textwidth}
        \includegraphics[width=\textwidth]{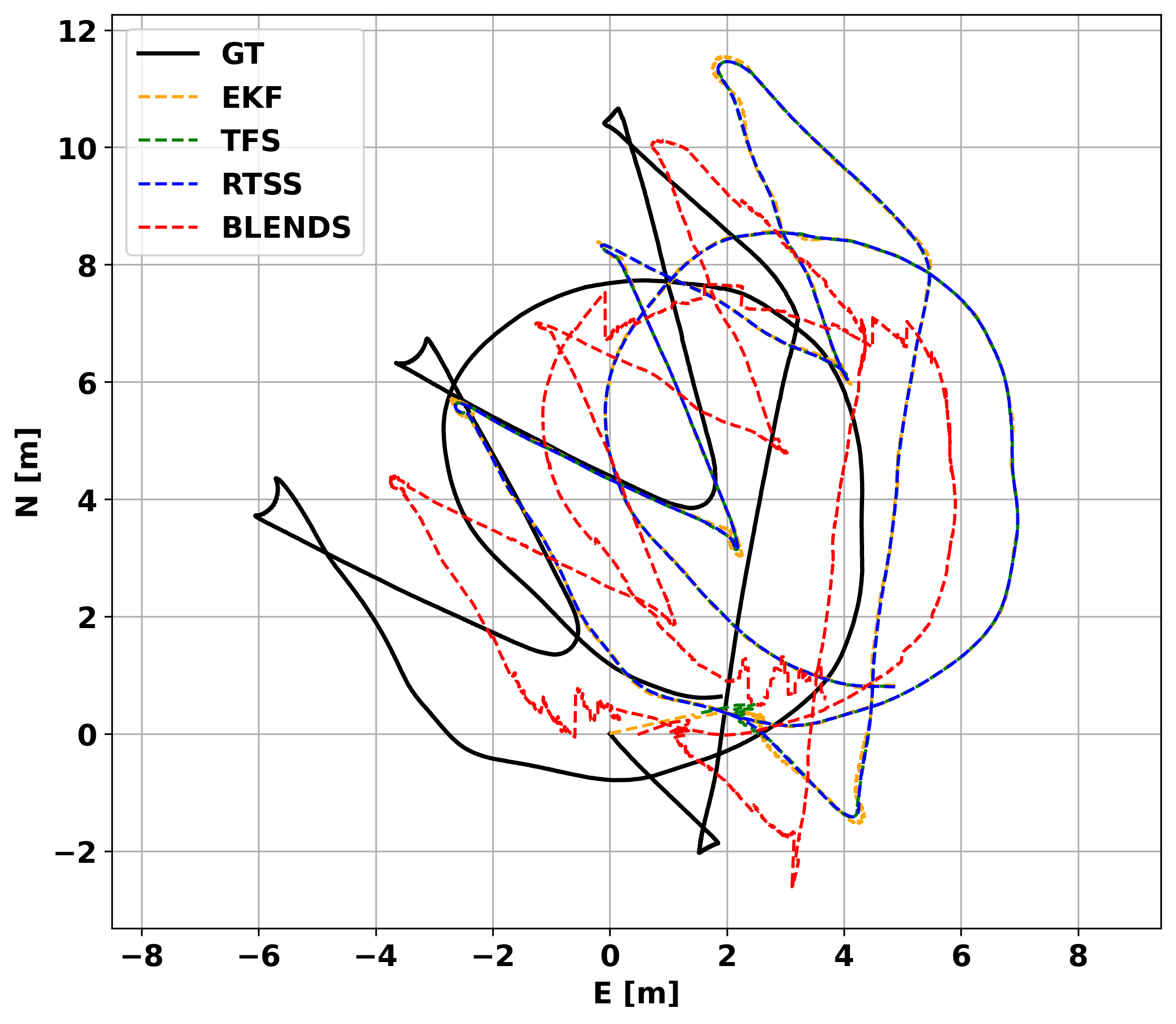}
        \caption{MR-Tes (Traj.~2).}
    \end{subfigure}\hfill
    \begin{subfigure}{0.32\textwidth}
        \includegraphics[width=\textwidth]{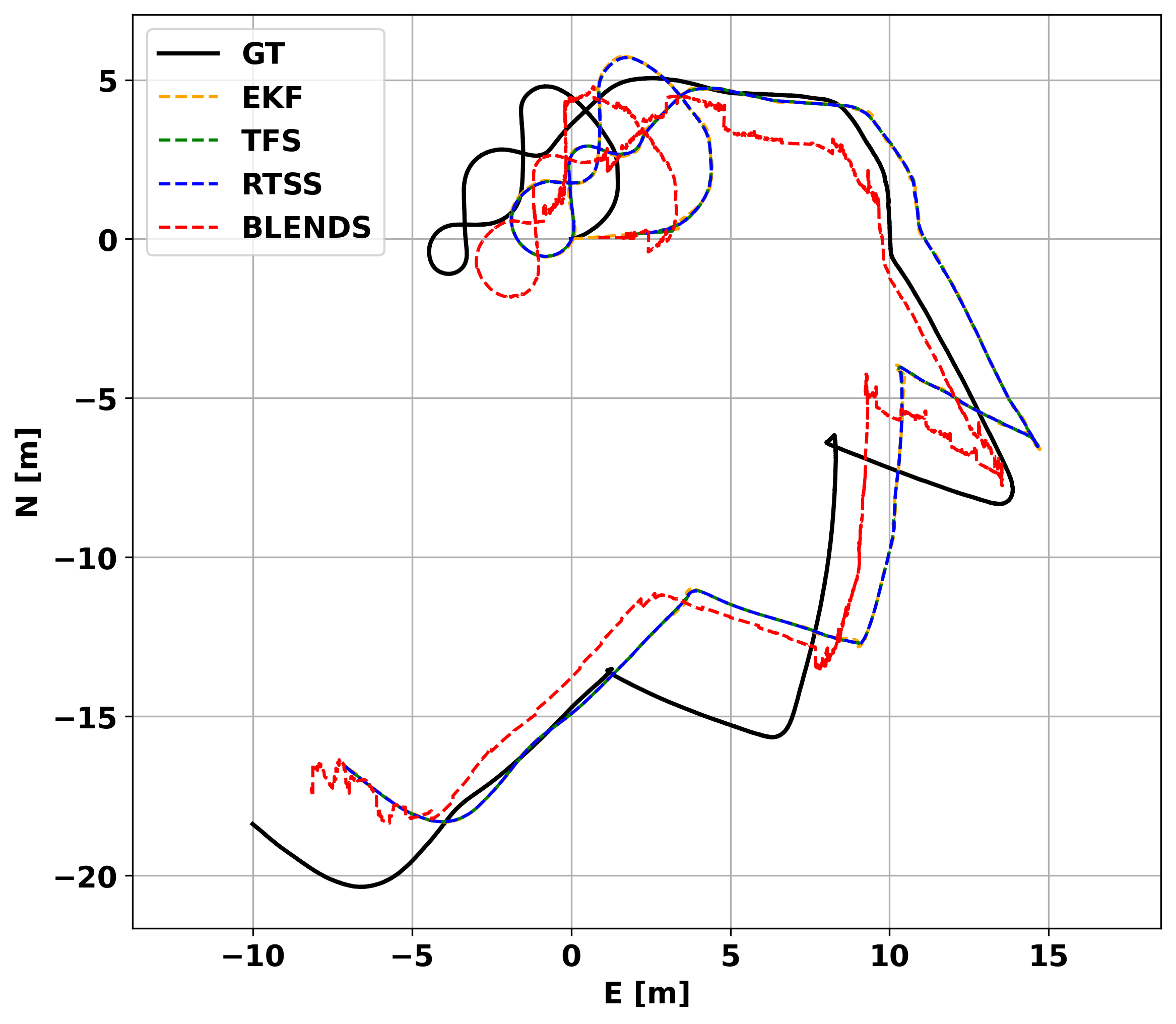}
        \caption{MR-Tes (Traj.~3).}
    \end{subfigure}
    \vspace{0.5em}
    \begin{subfigure}{0.32\textwidth}
        \includegraphics[width=\textwidth]{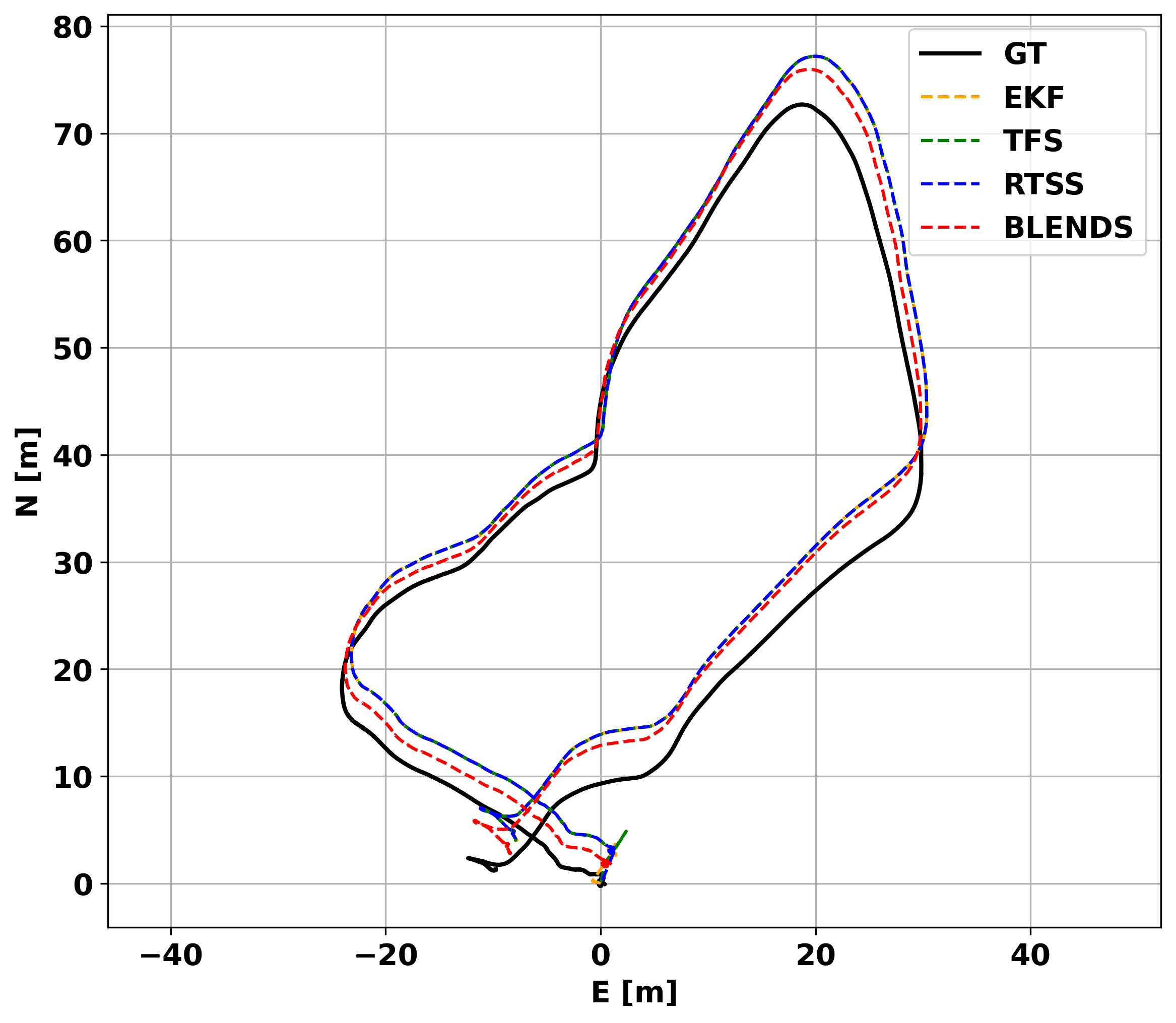}
        \caption{QR-Val (Traj.~2).}
    \end{subfigure}\hfill
    \begin{subfigure}{0.32\textwidth}
        \includegraphics[width=\textwidth]{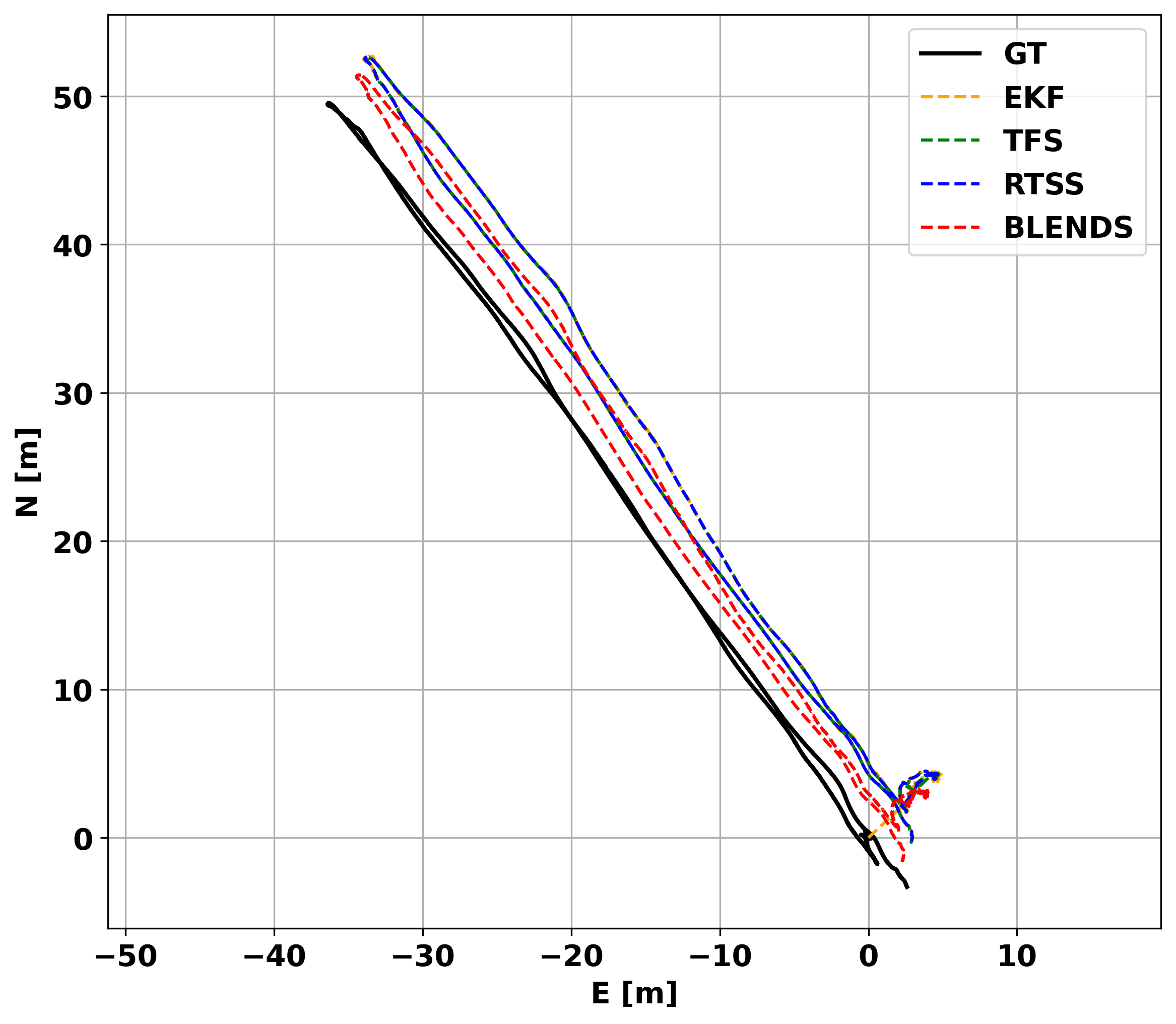}
        \caption{QR-Tes (Traj.~12).}
    \end{subfigure}\hfill
    \begin{subfigure}{0.32\textwidth}
        \includegraphics[width=\textwidth]{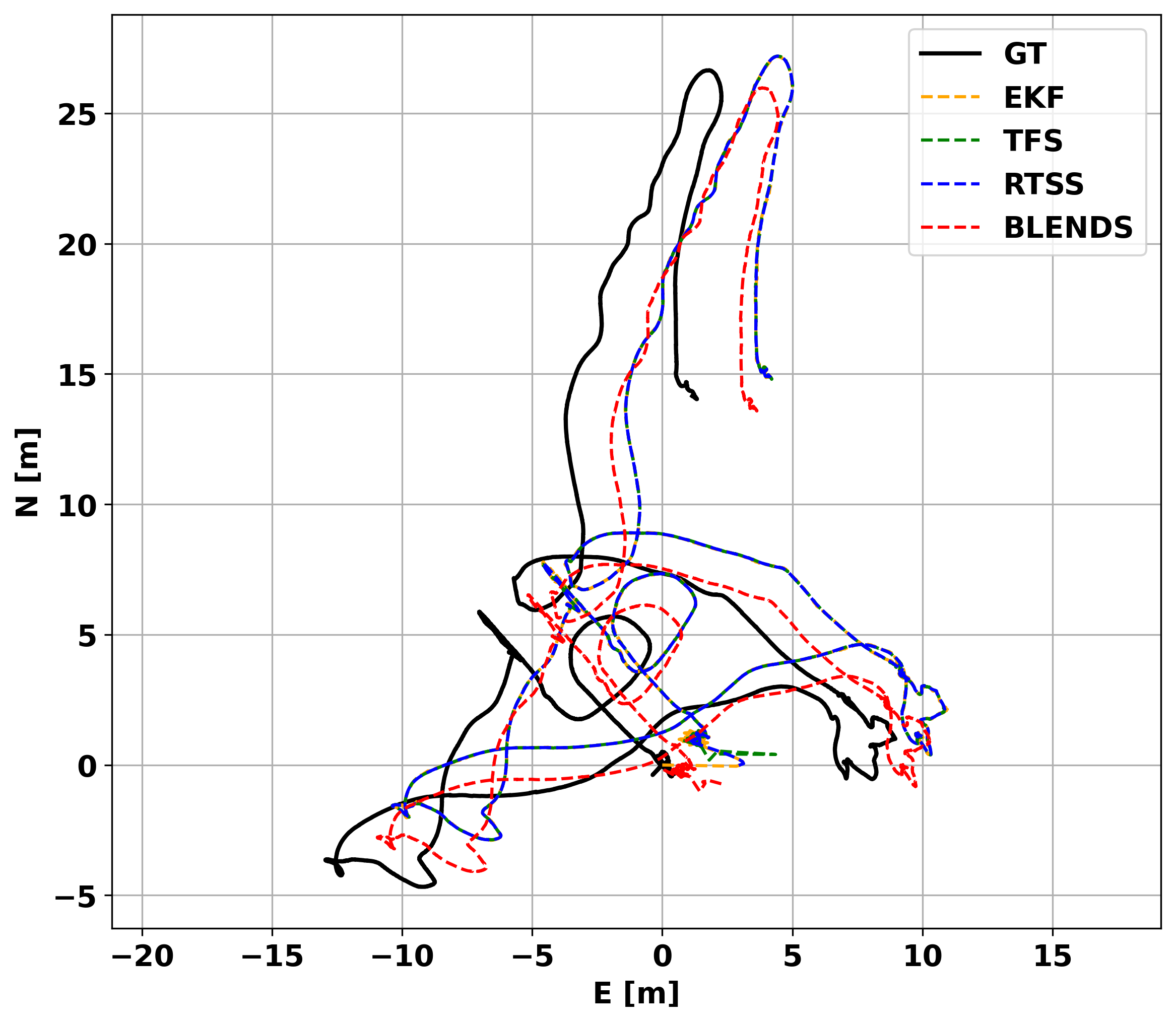}
        \caption{QR-Tes (Traj.~15).}
    \end{subfigure}
    \caption{2D position trajectory comparison across all evaluation 
    trajectories for the mobile robot (top) and quadrotor (bottom) datasets. 
    Each subplot shows the GT alongside the estimated 
    trajectories of EKF, TFS, RTSS, and BLENDS.}
    \label{fig:2d_trajectories}
\end{figure*}

\noindent
The GT labels used for supervised training are generated by 
applying the TFS to the inertial measurements along with 
the RTK position measurements. The resulting smoothed trajectory 
serves as the reference signal for all loss terms 
in~\eqref{eqn:total_loss}. This approach ensures that the network is 
trained to reproduce a theoretically optimal Bayesian estimate, rather 
than being directly supervised by raw RTK measurements alone.

\noindent
Training is performed using the AdamW optimizer 
\cite{loshchilov2017decoupled} with an initial learning rate of 
$10^{-2}$ and weight decay of $10^{-6}$, with a batch size of $128$ 
for both datasets. An adaptive learning rate scheduling is employed, 
reducing the learning rate by a factor of $0.1$ whenever the 
validation loss fails to improve for $10$ consecutive epochs, with a 
minimum learning rate floor of $10^{-8}$. The transformer encoder 
comprises $L=2$ layers with $H=16$ attention heads, a model 
dimension of $d_{\text{model}}=256$, a feed-forward dimension of 
$d_{\text{ff}}=512$, and a dropout rate of $0.1$. Both prediction 
heads use a hidden dimension of $d_{\text{head}}=256$. The input 
sequence is processed with a non-overlapping sliding window, where 
the window size equals the stride. For the mobile robot dataset, a window 
of $T=150$ time steps is used, corresponding to 1.5\,s of data at 
100\,Hz. For the quadrotor dataset, a window of $T=50$ time steps is used, 
corresponding to 0.25\,s of data at 200\,Hz. The curriculum ramp 
defined in~\eqref{eqn:ramp_fn} is applied with warm-up $e_w=1000$ 
epochs and power $p=2$, smoothly transitioning the correction bound 
from $\boldsymbol{m}^{\text{wide}}$ to $\boldsymbol{m}^{\text{base}}$ 
as specified in Table~\ref{tab:bounds}. 
\begin{table}[!h]
    \centering
    \caption{Per-state correction bounds for the BLENDS correction 
    term $\boldsymbol{c}_k$ for the mobile robot and quadrotor datasets. The 
    bound contracts during training from $\boldsymbol{m}^{\text{wide}}$ 
    to $\boldsymbol{m}^{\text{base}}$ via the curriculum ramp $\rho$.}
    \label{tab:bounds}
    \renewcommand{\arraystretch}{1.3}
    \resizebox{\columnwidth}{!}{%
    \begin{tabular}{lcccc}
        \hline
        & \multicolumn{2}{c}{\textbf{Mobile robot}} 
        & \multicolumn{2}{c}{\textbf{Quadrotor}} \\
        \cmidrule(lr){2-3}\cmidrule(lr){4-5}
        \textbf{State} & 
        $\boldsymbol{m}^{\text{wide}}$ & 
        $\boldsymbol{m}^{\text{base}}$ & 
        $\boldsymbol{m}^{\text{wide}}$ & 
        $\boldsymbol{m}^{\text{base}}$ \\
        \hline
        $\delta p_{x,y}$ [rad]            & $3\times10^{-7}$ & $2\times10^{-7}$ & $3\times10^{-7}$ & $2\times10^{-7}$ \\
        $\delta p_z$ [m]                  & $50.0$           & $1.0$            & $50.0$           & $1.0$            \\
        $\delta v_{x,y,z}$ [m/s]          & $2.0$            & $0.50$           & $15.0$           & $5.0$            \\
        $\delta\phi_{x,y,z}$ [rad]        & $\pi$            & $\pi/180$        & $\pi$            & $\pi/180$        \\
        $\delta b_{a_{x,y,z}}$ [m/s$^2$]  & $0.5$            & $0.2$            & $0.5$            & $0.1$            \\
        $\delta b_{g_{x,y,z}}$ [rad/s]    & $0.05$           & $0.002$          & $0.05$           & $0.01$           \\
        \hline
    \end{tabular}%
    }
\end{table}

\noindent
The covariance modification 
perturbation~\eqref{eqn:Dscale} is bounded by $\alpha=10^{-8}$. The BCL~\eqref{eqn:total_loss} is applied with weights $\lambda_p=10$, 
$\lambda_v=10^{-1}$, and $\lambda_c=10^{-2}$ for both datasets. The 
rotation weight differs between the two platforms: $\lambda_r=10^{-1}$ 
for the mobile robot dataset and $\lambda_r=1$ for the quadrotor dataset. Training was carried out for 200 epochs on the mobile robot dataset and 35 epochs on the quadrotor dataset. All 
of the above hyperparameters are shared between the two datasets, 
with the only dataset-specific differences being the window size and 
the correction bounds in Table~\ref{tab:bounds}, which are adapted 
to reflect the distinct motion characteristics of each platform.
\\ \noindent
Both the zero initialization of the output layer weights and biases, 
described above, and the correction bounds introduced 
in~\eqref{eqn:ramp} are critical for successful training. Together, 
they ensure that the network converges to reliable and physically 
meaningful values rather than diverging or causing training 
instability. The correction bounds for both datasets are designed to reflect the distinct 
motion characteristics of each platform. For the mobile robot dataset, the 
vehicle is operating at low speeds with smooth, planar 
trajectories, and consequently the velocity bounds are set 
conservatively ($\boldsymbol{m}^{\text{wide}} = 2.0$\,m/s, 
$\boldsymbol{m}^{\text{base}} = 0.5$\,m/s). In contrast, the quadrotor 
dataset involves performing agile three-dimensional 
manoeuvres at significantly higher velocities, requiring considerably 
wider bounds ($\boldsymbol{m}^{\text{wide}} = 15.0$\,m/s, 
$\boldsymbol{m}^{\text{base}} = 5.0$\,m/s) to allow the network 
sufficient freedom to correct the larger dynamic errors encountered 
during flight. The remaining state bounds are shared across both 
datasets. It should be noted that these bounds are themselves 
hyperparameters that must be tuned through a trial-and-error process, 
guided by prior knowledge of the platform dynamics and the expected 
ranges of each navigation state.
\\ \noindent
A final aspect that must be addressed is whether each component of 
the BCL converges independently during training. While the overall 
loss convergence is important, it does not tell the full story, since 
the total loss is dominated by the position term due to its relatively 
large weight. Therefore, each component must be monitored 
individually to ensure that the network is learning meaningful corrections across all navigation states. 
\subsection{Results}
\noindent
The performance of BLENDS is evaluated against three classical 
baselines, the forward EKF, the TFS, and the RTSS, across both datasets. For the mobile robot dataset, MR-Val consists of Trajectory~1, while 
MR-Tes includes Trajectory~2 and Trajectory~3. For the quadrotor 
dataset, QR-Val consists of Trajectory~2, while QR-Tes includes 
Trajectory~12 and Trajectory~15. Results, for the two platforms, are reported in 
Tables~\ref{tab:arazim_results} and~\ref{tab:mars_results}.
\begin{table}[!h]
\centering
\caption{Mobile robot dataset per-axis RMSE for position, velocity, 
and orientation across validation and test trajectories. }
\label{tab:arazim_results}
\resizebox{\columnwidth}{!}{%
\begin{tabular}{llccccccccc}
\toprule
& & \multicolumn{3}{c}{\textbf{Position [m]}} 
  & \multicolumn{3}{c}{\textbf{Velocity [m/s]}} 
  & \multicolumn{3}{c}{\textbf{Orientation [rad]}} \\
\cmidrule(lr){3-5}\cmidrule(lr){6-8}\cmidrule(lr){9-11}
\textbf{Traj.} & \textbf{Method} 
  & $p_N$ & $p_E$ & $p_D$ 
  & $v_N$ & $v_E$ & $v_D$ 
  & $\phi$ & $\theta$ & $\psi$ \\
\midrule
\multirow{4}{*}{\shortstack{MR-Val\\(1)}}
  & EKF    & 0.824 & 2.373 & 0.953 & 0.195 & 0.247 & 0.139 & 0.027 & 0.027 & 1.795 \\
  & RTSS   & 0.824 & 2.373 & \textbf{0.951} & \textbf{0.146} & \textbf{0.148} & \textbf{0.114} & \textbf{0.017} & \textbf{0.018} & \textbf{1.777} \\
  & TFS    & 0.825 & 2.374 & \textbf{0.951} & 0.172 & 0.206 & 0.137 & 0.018 & 0.019 & 1.786 \\
  & BLENDS & \textbf{0.793} & \textbf{1.335} & 0.958 & 0.177 & 0.198 & 0.137 & 0.018 & 0.019 & 1.786 \\
\midrule
\multirow{4}{*}{\shortstack{MR-Tes\\(2)}}
  & EKF    & 1.182 & 2.729 & 1.403 & 0.172 & 0.220 & 0.130 & 0.039 & 0.018 & 2.099 \\
  & RTSS   & 1.179 & 2.729 & 1.402 & \textbf{0.142} & \textbf{0.140} & \textbf{0.119} & \textbf{0.036} & \textbf{0.017} & 2.306 \\
  & TFS    & 1.179 & 2.729 & 1.402 & 0.152 & 0.181 & 0.126 & 0.037 & 0.018 & 2.322 \\
  & BLENDS & \textbf{0.543} & \textbf{1.661} & \textbf{1.345} & 0.152 & 0.177 & 0.127 & 0.037 & 0.018 & \textbf{2.099} \\
\midrule
\multirow{4}{*}{\shortstack{MR-Tes\\(3)}}
  & EKF    & 1.623 & 2.350 & 1.207 & 0.141 & 0.242 & 0.149 & 0.025 & 0.023 & 2.580 \\
  & RTSS   & 1.620 & 2.349 & \textbf{1.206} & 0.129 & 0.222 & 0.146 & 0.017 & \textbf{0.018} & 2.548 \\
  & TFS    & 1.620 & 2.348 & \textbf{1.206} & \textbf{0.123} & \textbf{0.159} & \textbf{0.134} & \textbf{0.015} & \textbf{0.018} & \textbf{2.534} \\
  & BLENDS & \textbf{1.446} & \textbf{1.368} & 1.242 & 0.129 & 0.220 & 0.149 & 0.017 & \textbf{0.018} & 2.549 \\
\bottomrule
\end{tabular}%
}
\end{table}

\noindent
As shown in the tables, BLENDS achieves consistent and substantial improvements in horizontal 
position accuracy across both datasets and all unseen test 
trajectories. On the mobile robot dataset, BLENDS reduces the horizontal position 
error on the validation trajectory and consistently on the two unseen 
test trajectories. As illustrated in Fig.~\ref{fig:2d_trajectories}, 
the estimated trajectory of BLENDS visually aligns more closely with 
the GT in the North--East plane compared to the classical 
baselines. On MR-Tes trajectory~2, the $p_N$ improvement stands at 
approximately 54\% while the $p_E$ RMSE improvement is around 40\% 
relative to the EKF baseline. A similar performance gain is observed 
on MR-Tes trajectory~3, where $p_N$ improves by more than 10\% and 
$p_E$ by more than 40\%, as further quantified in 
Fig.~\ref{fig:arazim_improvement}. 
The classical smoothers TFS and RTSS provide negligible improvement 
in position, as the positional bias cannot be compensated by 
covariance-based fusion alone, as established in 
Section~\ref{sec:motivation}, in other words, there is insufficient 
information within the filter framework to correct the systematic 
offset. BLENDS, through its learned covariance modification and 
additive correction term, is able to identify and compensate for the 
systematic position bias introduced by the loosely coupled GNSS 
position measurements.\\ \noindent
\begin{table}[!b]
\centering
\caption{Quadrotor dataset per-axis RMSE for position, velocity, 
and orientation across validation and test trajectories. }
\label{tab:mars_results}
\resizebox{\columnwidth}{!}{%
\begin{tabular}{llccccccccc}
\toprule
& & \multicolumn{3}{c}{\textbf{Position [m]}} 
  & \multicolumn{3}{c}{\textbf{Velocity [m/s]}} 
  & \multicolumn{3}{c}{\textbf{Orientation [rad]}} \\
\cmidrule(lr){3-5}\cmidrule(lr){6-8}\cmidrule(lr){9-11}
\textbf{Traj.} & \textbf{Method} 
  & $p_N$ & $p_E$ & $p_D$ 
  & $v_N$ & $v_E$ & $v_D$ 
  & $\phi$ & $\theta$ & $\psi$ \\
\midrule
\multirow{4}{*}{\shortstack{QR-Val\\(2)}}
  & EKF    & 3.703 & 1.100 & 3.819 & 0.287 & 0.175 & \textbf{0.275} & 0.041 & 0.024 & 0.408 \\
  & RTSS   & 3.704 & 1.101 & 3.822 & 0.283 & 0.166 & 0.292 & 0.032 & 0.012 & \textbf{0.179} \\
  & TFS    & 3.704 & 1.102 & 3.822 & 0.372 & 0.176 & 0.369 & 0.028 & \textbf{0.011} & 0.183 \\
  & BLENDS & \textbf{2.533} & \textbf{0.586} & \textbf{3.801} & \textbf{0.184} & \textbf{0.118} & 0.276 & \textbf{0.013} & 0.012 & 0.426 \\
\midrule
\multirow{4}{*}{\shortstack{QR-Tes\\(12)}}
  & EKF    & 3.423 & 2.609 & 2.979 & 0.205 & 0.281 & 0.255 & 0.037 & 0.050 & \textbf{0.406} \\
  & RTSS   & 3.422 & 2.611 & \textbf{2.977} & 0.147 & 0.165 & 0.273 & \textbf{0.034} & \textbf{0.036} & 0.558 \\
  & TFS    & 3.422 & 2.610 & \textbf{2.977} & \textbf{0.142} & \textbf{0.160} & \textbf{0.225} & 0.036 & 0.038 & 0.542 \\
  & BLENDS & \textbf{2.226} & \textbf{2.117} & 3.001 & 0.159 & 0.162 & 0.283 & \textbf{0.034} & 0.053 & 0.659 \\
\midrule
\multirow{4}{*}{\shortstack{QR-Tes\\(15)}}
  & EKF    & 1.345 & 2.298 & 1.808 & 0.152 & 0.213 & 0.200 & 0.037 & 0.023 & \textbf{0.224} \\
  & RTSS   & 1.345 & 2.298 & 1.808 & 0.127 & 0.131 & 0.207 & 0.027 & \textbf{0.013 }& 0.263 \\
  & TFS    & 1.345 & 2.298 & 1.807 & \textbf{0.122} & \textbf{0.119} & \textbf{0.177} & \textbf{0.024} & 0.015 & 0.243 \\
  & BLENDS & \textbf{0.496} & \textbf{1.752} & \textbf{1.787} & 0.150 & 0.134 & 0.209 & 0.028 & 0.018 & 0.257 \\
\bottomrule
\end{tabular}%
}
\end{table}
These results generalize to the quadrotor dataset, where the 
horizontal position improvements are even more pronounced. The 
visual separation between the BLENDS trajectory and those of the 
classical baselines in the North--East plane is clearly visible in 
Fig.~\ref{fig:2d_trajectories}. On QR-Tes trajectory~12, BLENDS 
reduces $p_N$ by up to 34\% and $p_E$ by up to 18\% relative to the 
EKF. On QR-Tes trajectory~15, which was recorded in an entirely 
unseen zone of the dataset, BLENDS achieves reductions of more than 
63\% in $p_N$ and more than 23\% in $p_E$, as further quantified in 
Fig.~\ref{fig:mars_improvement}. The classical smoothers again show 
minimal position improvement across all quadrotor trajectories, further 
confirming that the systematic bias cannot be corrected through 
covariance-based fusion alone and that a data-driven correction 
mechanism is necessary.
\begin{figure}[!h]
    \centering
    \begin{subfigure}{\columnwidth}
        \includegraphics[width=\columnwidth]{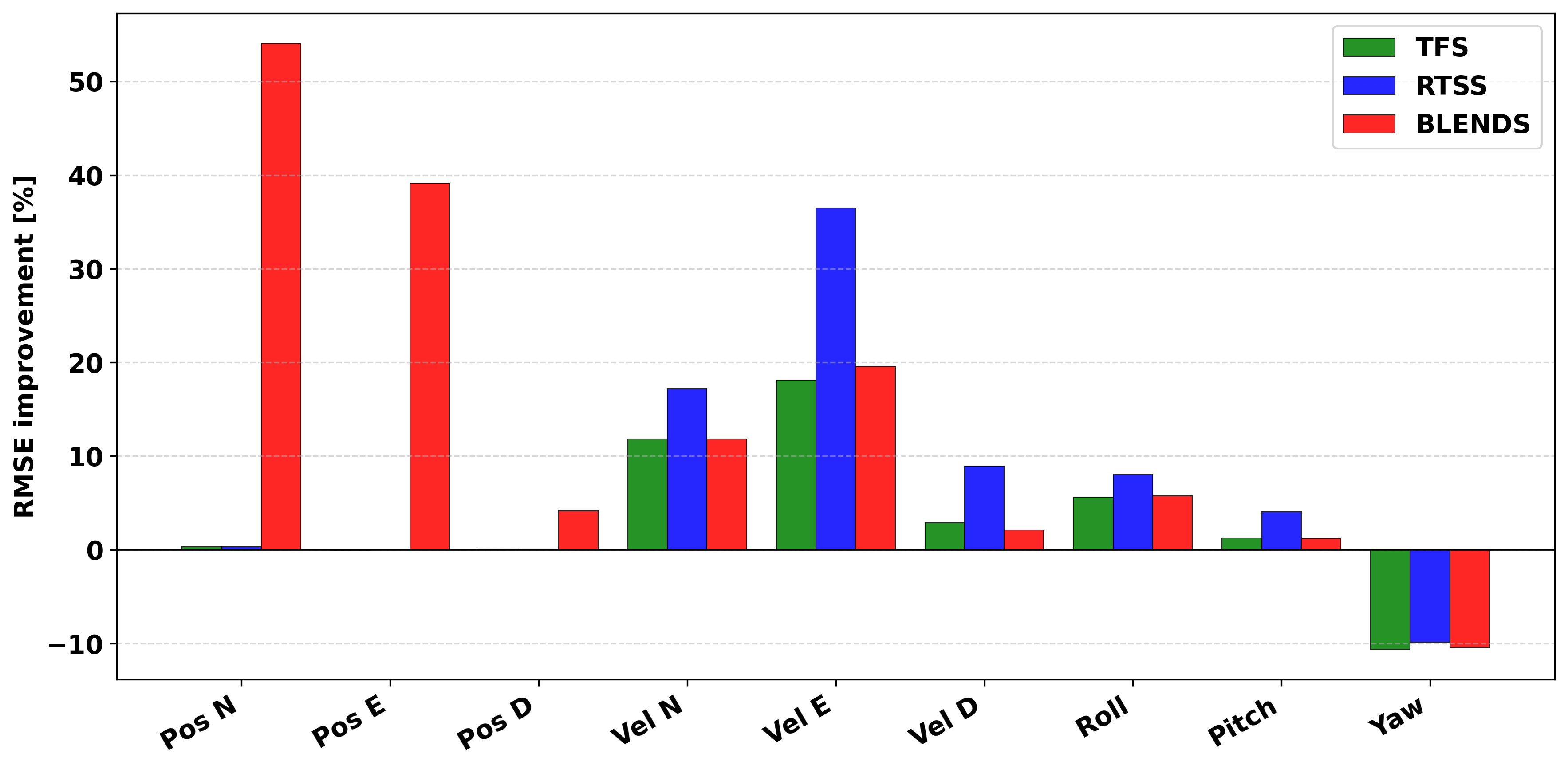}
        \caption{MR-Tes (Traj.~2).}
    \end{subfigure}
    \vspace{0.5em}
    \begin{subfigure}{\columnwidth}
        \includegraphics[width=\columnwidth]{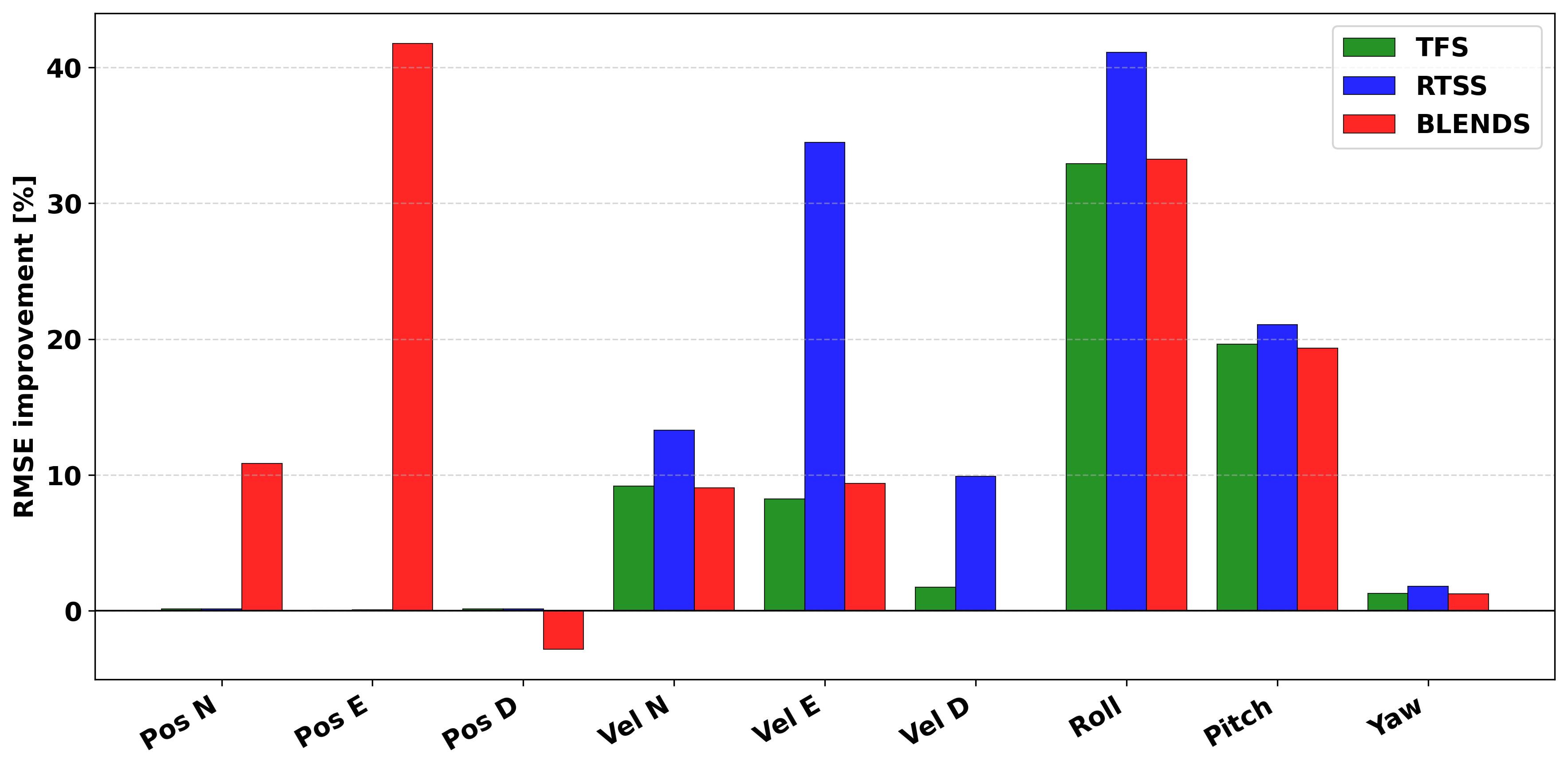}
        \caption{MR-Tes (Traj.~3).}
    \end{subfigure}
    \caption{RMSE improvement [\%] relative to the forward EKF 
    for TFS, RTSS, and BLENDS across all navigation states 
    on the mobile robot dataset test trajectories. Positive values indicate 
    improvement over the EKF baseline.}
    \label{fig:arazim_improvement}
\end{figure}
\begin{figure}[!b]
    \centering
    \begin{subfigure}{\columnwidth}
        \includegraphics[width=\columnwidth]{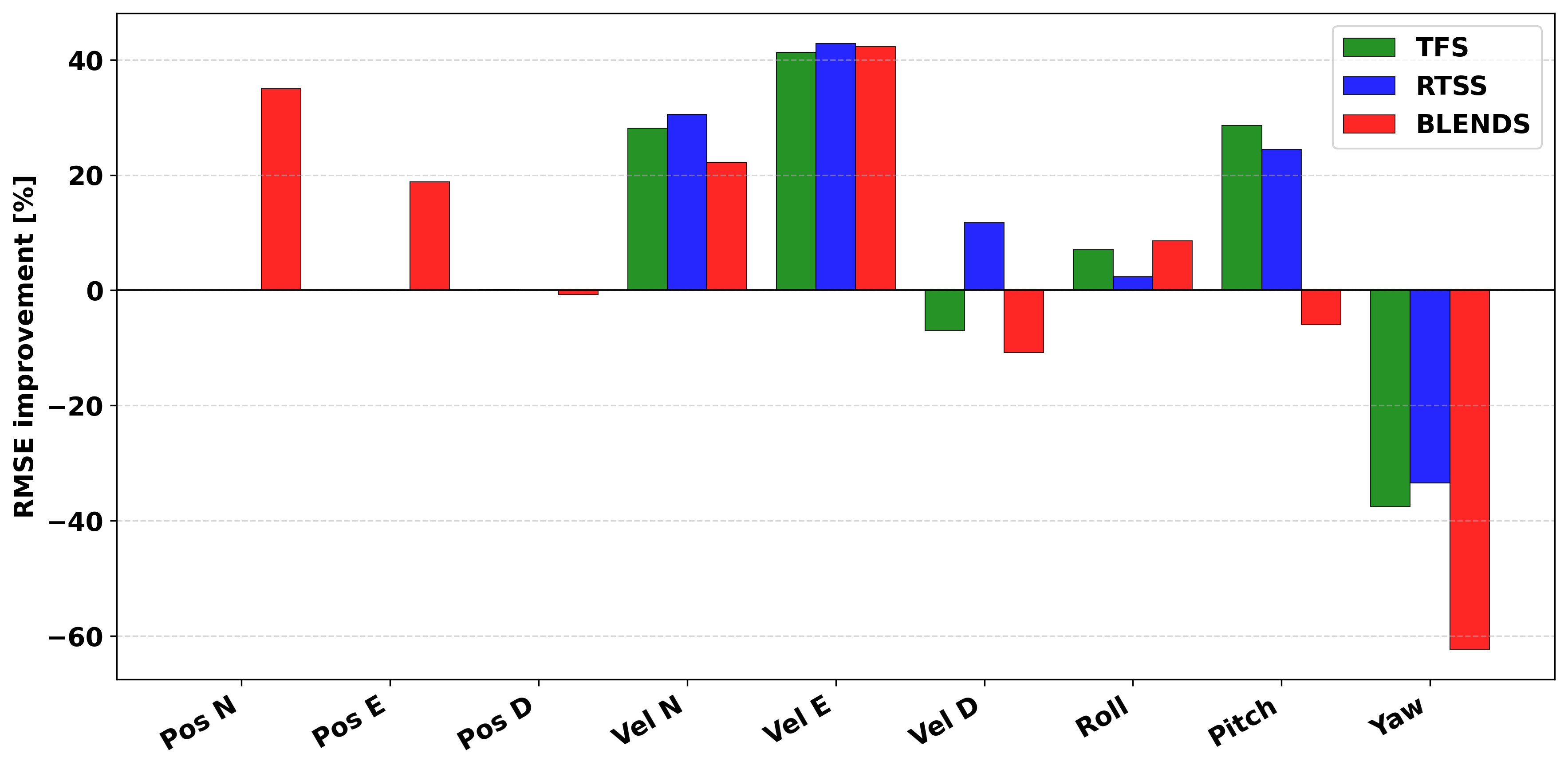}
        \caption{QR-Tes (Traj.~12).}
    \end{subfigure}
    \vspace{0.5em}
    \begin{subfigure}{\columnwidth}
        \includegraphics[width=\columnwidth]{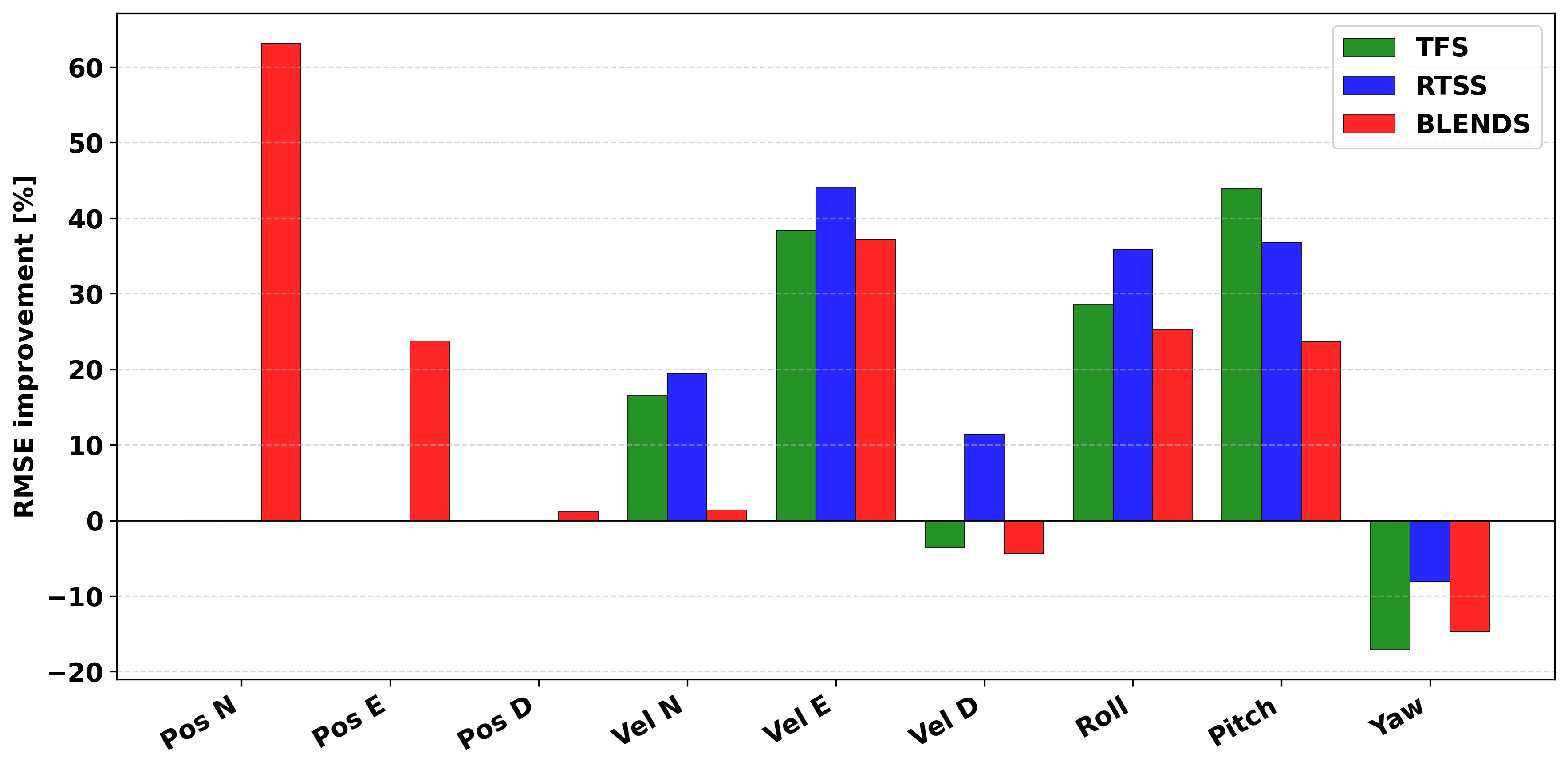}
        \caption{QR-Tes (Traj.~15).}
    \end{subfigure}
    \caption{RMSE improvement [\%] relative to the forward EKF 
    for TFS, RTSS, and BLENDS across all navigation states 
    on the quadrotor dataset test trajectories. Positive values indicate 
    improvement over the EKF baseline.}
    \label{fig:mars_improvement}
\end{figure}
\noindent
For velocity and orientation, the classical smoothers TFS and RTSS 
generally outperform BLENDS. On both the mobile robot and quadrotor datasets, RTSS consistently achieves the lowest velocity RMSE 
across all axes and trajectories, with improvements of over 30\% 
relative to the EKF on the mobile robot data and improvements of more 
than 40\% on the quadrotor data, while BLENDS shows comparable but 
slightly higher velocity errors. A similar pattern holds for 
orientation, where TFS and RTSS reduce roll and pitch errors more 
effectively than BLENDS. This behavior is expected, as the classical 
smoothers are theoretically optimal for Gaussian linear systems and 
seek to minimize the estimation variance. Both velocity and 
orientation are not directly measured by GNSS and are instead deduced 
by the filter through the INS mechanism and state propagation. As 
a result, they are more responsive to changes in filter parameters and 
to the incorporation of additional information, such as the future 
measurements provided by the smoother in the backward pass. The 
learned correction in BLENDS primarily targets the systematic position 
bias rather than the stochastic velocity and orientation errors. 
Although the BCL includes a minimum-variance 
regularization term, it is not absolute and permits the other loss 
components, particularly the position supervision, to trade off 
some velocity and orientation accuracy in favor of bias correction. 
In contrast, the model-based smoothers exclusively target variance 
reduction, which makes them more effective for these states.
\begin{figure*}[!h]
    \centering
    \begin{subfigure}{0.32\textwidth}
        \includegraphics[width=\textwidth]{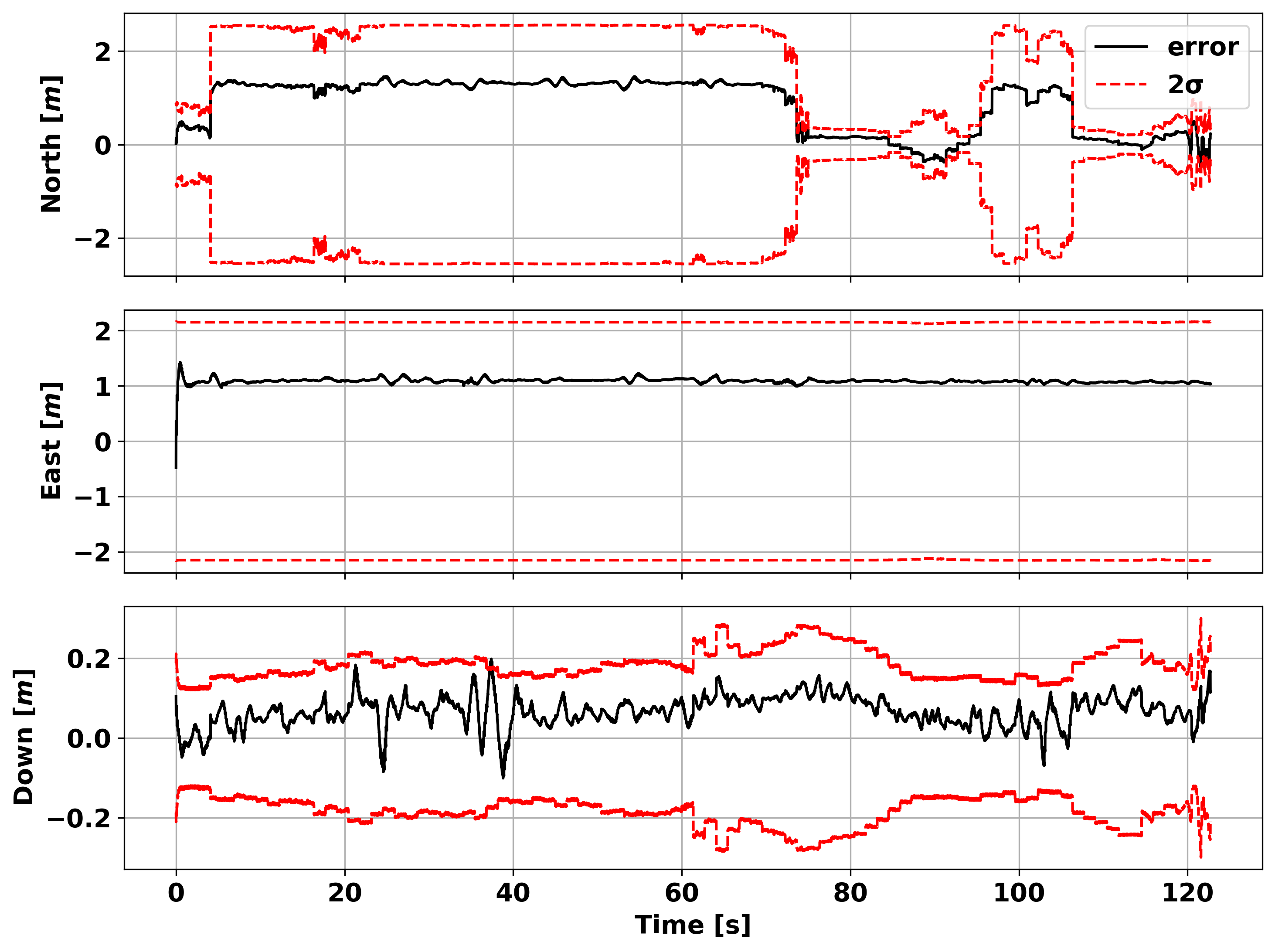}
        \caption{MR-Tes (Traj.~2) position error.}
    \end{subfigure}\hfill
    \begin{subfigure}{0.32\textwidth}
        \includegraphics[width=\textwidth]{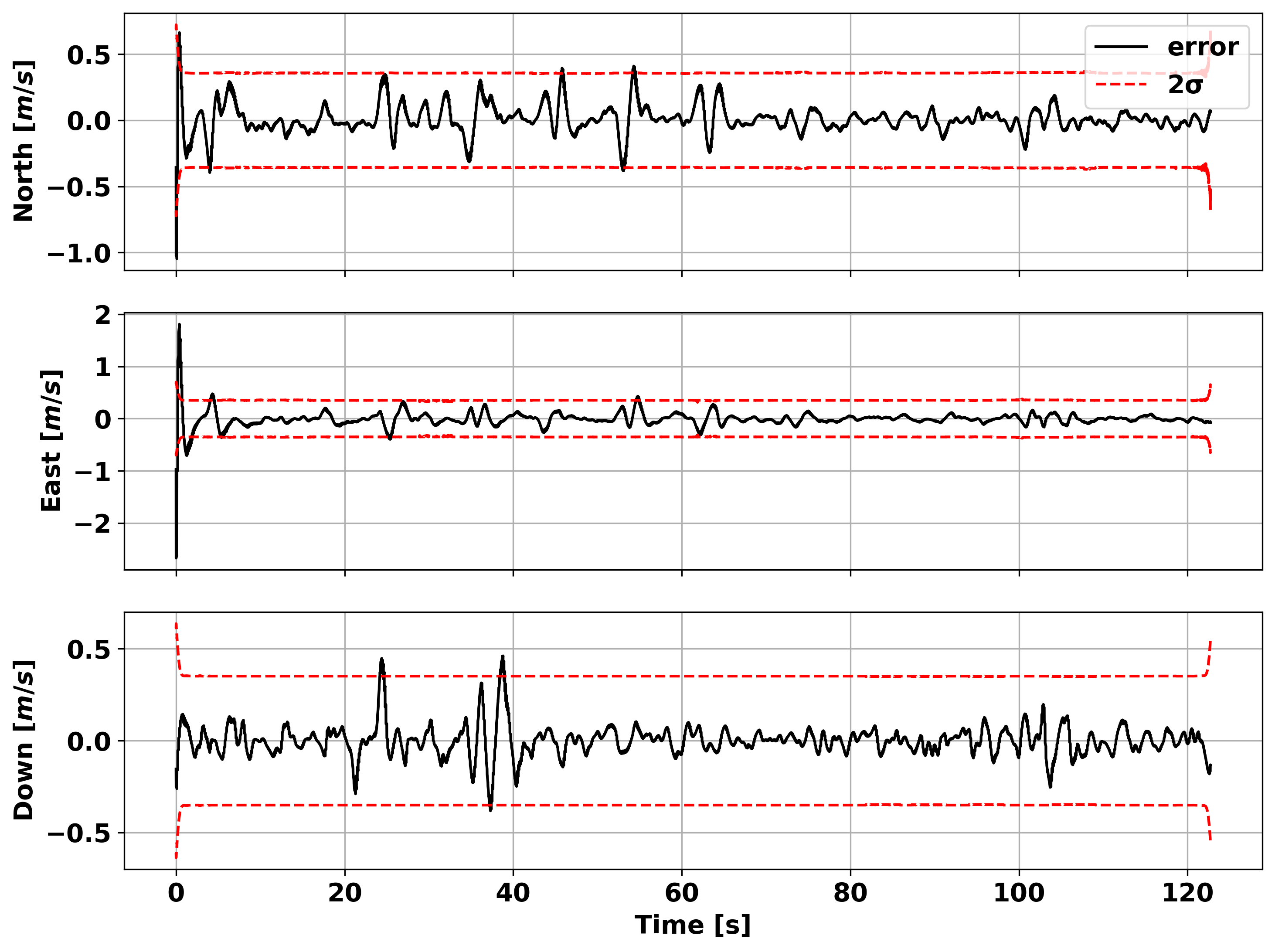}
        \caption{MR-Tes (Traj.~2) velocity error.}
    \end{subfigure}\hfill
    \begin{subfigure}{0.32\textwidth}
        \includegraphics[width=\textwidth]{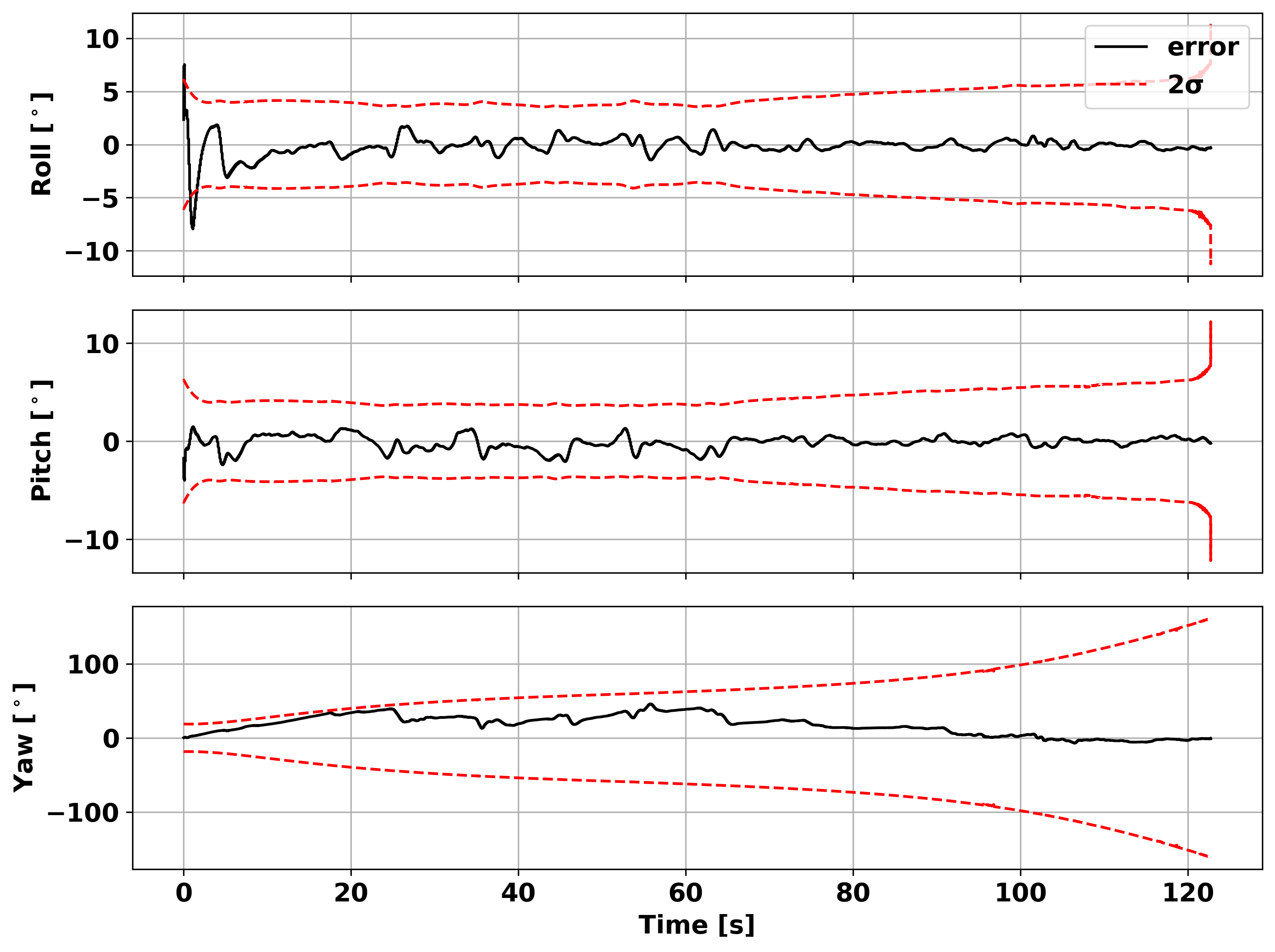}
        \caption{MR-Tes (Traj.~2) orientation error.}
        \label{fig:covariance_plots_mr}
    \end{subfigure}
    \vspace{0.5em}
    \begin{subfigure}{0.32\textwidth}
        \includegraphics[width=\textwidth]{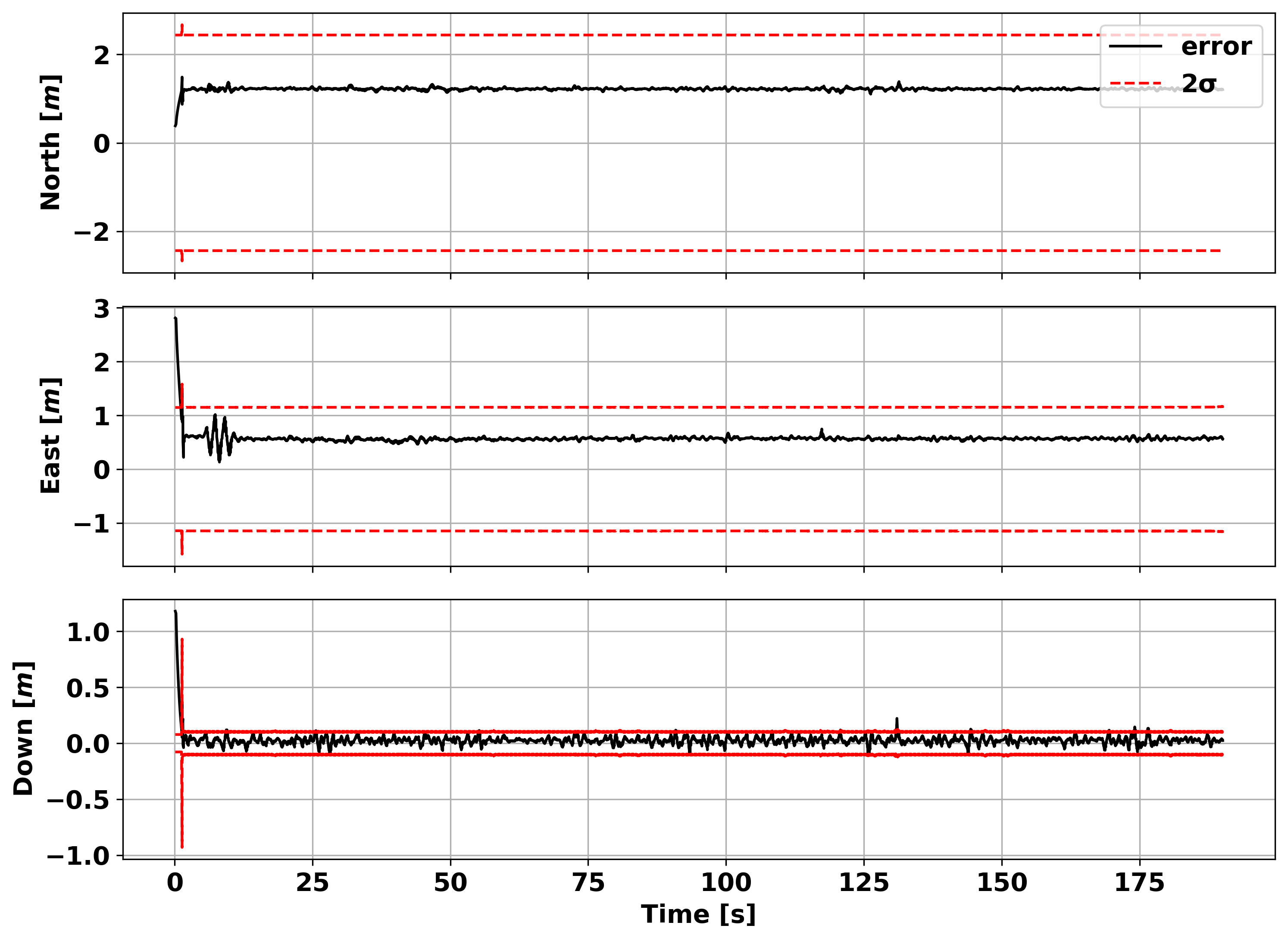}
        \caption{QR-Tes (Traj.~15) position error.}
    \end{subfigure}\hfill
    \begin{subfigure}{0.32\textwidth}
        \includegraphics[width=\textwidth]{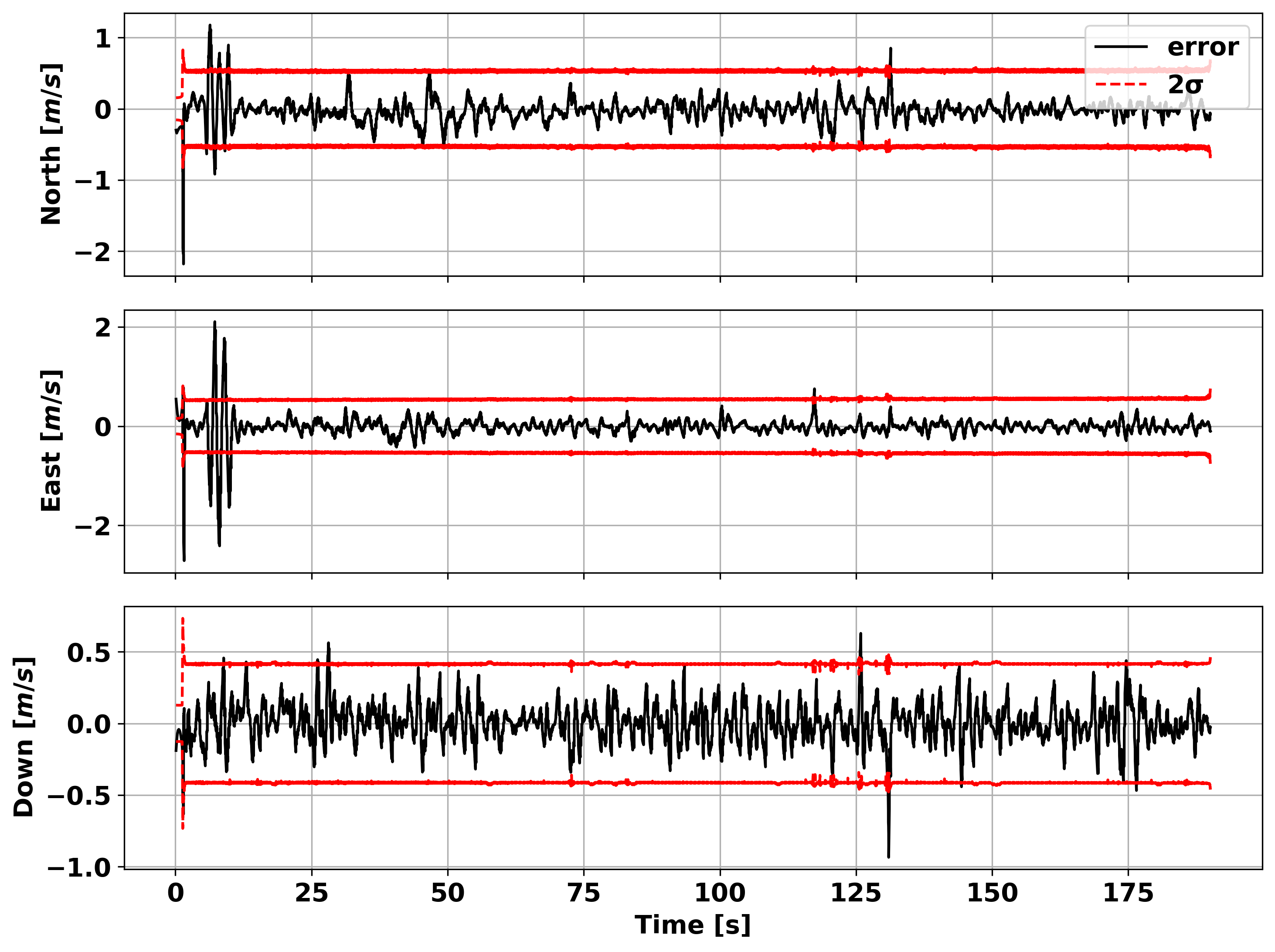}
        \caption{QR-Tes (Traj.~15) velocity error.}
    \end{subfigure}\hfill
    \begin{subfigure}{0.32\textwidth}
        \includegraphics[width=\textwidth]{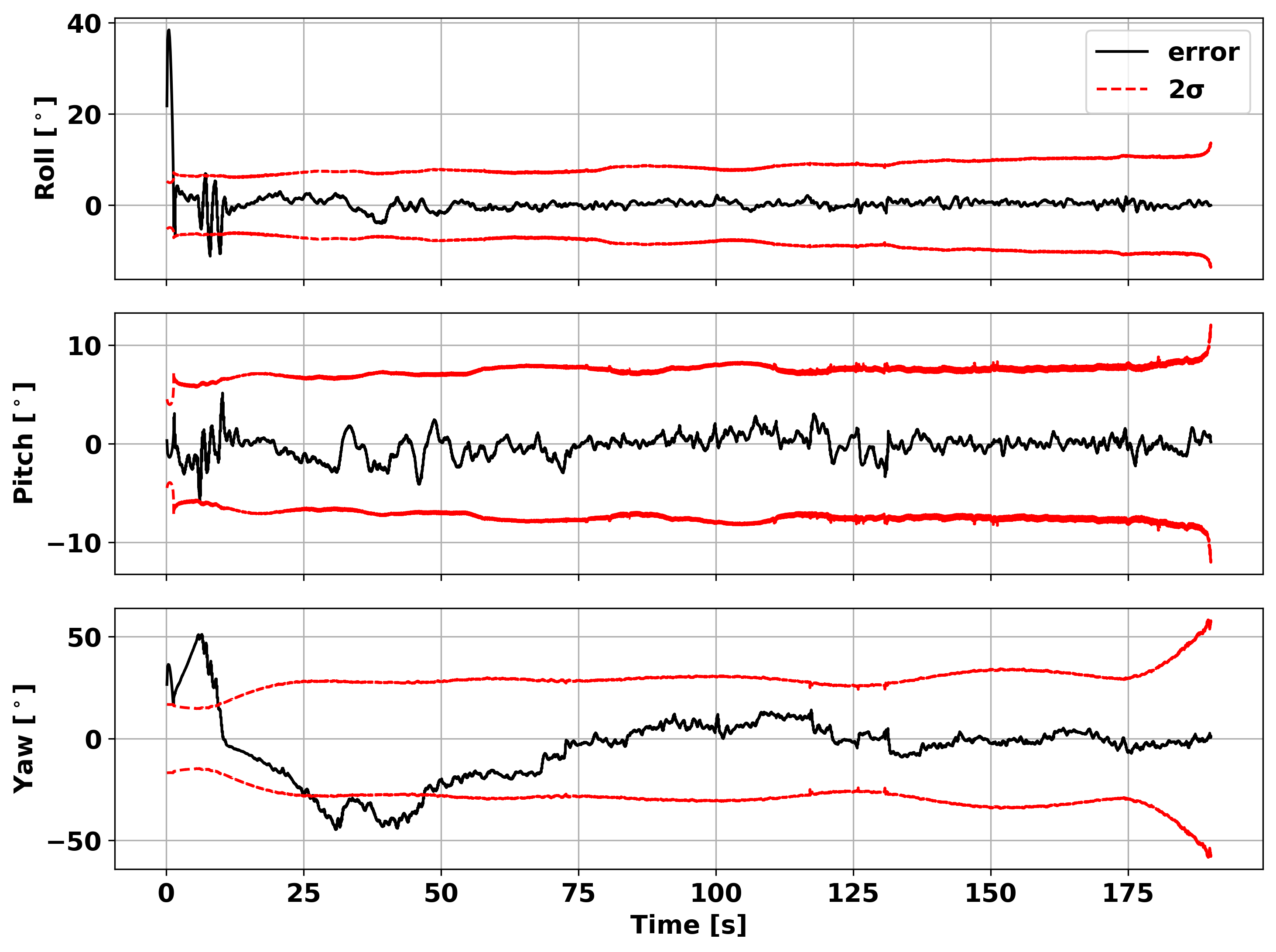}
        \caption{QR-Tes (Traj.~15) orientation error.}
        \label{fig:covariance_plots_qr}
    \end{subfigure}
    \caption{BLENDS navigation error and $2\sigma$ covariance bounds 
    for a representative test trajectory from the mobile robot (top) and 
    quadrotor (bottom) datasets. Each subplot shows the per-axis error 
    (black) and the corresponding $\pm2\sigma$ uncertainty envelope 
    (red dashed), demonstrating that the smoothed error remains 
    consistent with the predicted covariance throughout the 
    trajectory.}
    \label{fig:covariance_plots}
\end{figure*}
\\ \noindent
An exception is observed for yaw, where all methods, including the 
classical smoothers, exhibit degraded performance or even 
regression relative to the EKF on certain trajectories. This is a 
known consequence of the fact that heading is not observable 
in a loosely coupled INS/GNSS system without a dual-antenna 
configuration or magnetometer aiding~\cite{teodori2020analysis}. As a result, the yaw error 
grows unbounded even in the presence of GNSS position updates. The $2\sigma$ envelope in 
Figs.~\ref{fig:covariance_plots}(c) and~\ref{fig:covariance_plots}(f) 
reflects this diverging uncertainty and explains why smoothing the 
heading does not yield a meaningful improvement, once the yaw estimate 
has drifted significantly from its true value, neither the forward nor 
the backward filter retains sufficient heading information to recover 
it, and the smoother merely inherits the diverged estimate from both 
passes.
\\ \noindent
Of particular interest are the north and east position error states, 
shown in Figs.~\ref{fig:covariance_plots}(a) 
and~\ref{fig:covariance_plots}(d), where BLENDS expands the $2\sigma$ 
envelope relative to the classical smoothers. This is consistent with 
the presence of a systematic position offset, by incorporating the 
learned correction, the smoother explicitly acknowledges and expresses 
the residual uncertainty associated with the bias. For the mobile 
robot dataset, an even more notable behavior is observed in 
Fig.~\ref{fig:covariance_plots}(a): the envelope inflates and 
contracts adaptively in response to the vehicle's manoeuvres, with 
the error signal following accordingly. This reflects the network's 
ability to modulate its uncertainty estimates based on the local 
dynamics of the trajectory. For all remaining states, the error 
remains well bounded within the predicted envelope, 95\% confidence interval, with only a small 
number of outlier estimates, which is closely consistent with 
Bayesian estimation theory. 
\\ \noindent
An additional way to evaluate a smoother is by quantifying how much 
the uncertainty is reduced relative to the forward EKF. For this 
purpose, the PCI~\eqref{eqn:pci} over time is shown in 
Figs.~\ref{fig:arazim_pci} and~\ref{fig:mars_pci}. BLENDS maintains 
a consistently higher PCI than both TFS and RTSS across all test 
trajectories on both datasets, indicating that the learned covariance 
modification produces tighter smoothed covariances without sacrificing 
consistency. On the quadrotor dataset, BLENDS sustains a PCI of 
approximately 60--70\% throughout the majority of both test 
trajectories, compared to 50--60\% for the classical smoothers, with 
the gap widening towards the end of each trajectory. For the mobile robot dataset, the PCI values of BLENDS are 
comparable to those of the classical smoothers, which at first glance 
may appear to suggest a lesser degree of uncertainty reduction. 
However, this result should be interpreted carefully. As discussed 
above, BLENDS deliberately inflates the position uncertainty envelope 
to account for the systematic offset, thereby increasing the trace of 
the smoothed covariance in the position states. Since the PCI is 
computed over the full state covariance trace, this inflation in the 
position block partially offsets the uncertainty reduction achieved 
in the remaining states. Consequently, PCI values similar to those of classical smoothers indicate that BLENDS achieves a meaningful reduction in uncertainty across velocity, orientation, and bias states. At the same time, it provides a more honest and informative representation of the position uncertainty, consistent with the principles of Bayesian estimation.
\\ \noindent
A key result of this evaluation is that BLENDS generalizes across two 
substantially different platforms with distinct dynamics and operating 
speeds, a mobile robot and a quadrotor, using the same 
network architecture and training procedure, with only the correction 
bounds and window duration adapted to reflect the characteristics of 
each platform. The consistent horizontal position improvement across 
all four unseen test trajectories demonstrates that the learned 
correction captures a systematic bias that is present across different 
trajectory types and motion profiles, rather than overfitting to the 
specific patterns seen during training. This cross-platform 
generalization, combined with the covariance consistency demonstrated 
by the $2\sigma$ envelopes and the PCI results, confirms that BLENDS 
operates within the theoretical foundations of Bayesian estimation 
while meaningfully extending the capabilities of classical smoothing 
algorithms in the presence of biased GNSS measurements.
\begin{figure}[!t]
    \centering
    \begin{subfigure}{\columnwidth}
        \includegraphics[width=\columnwidth]{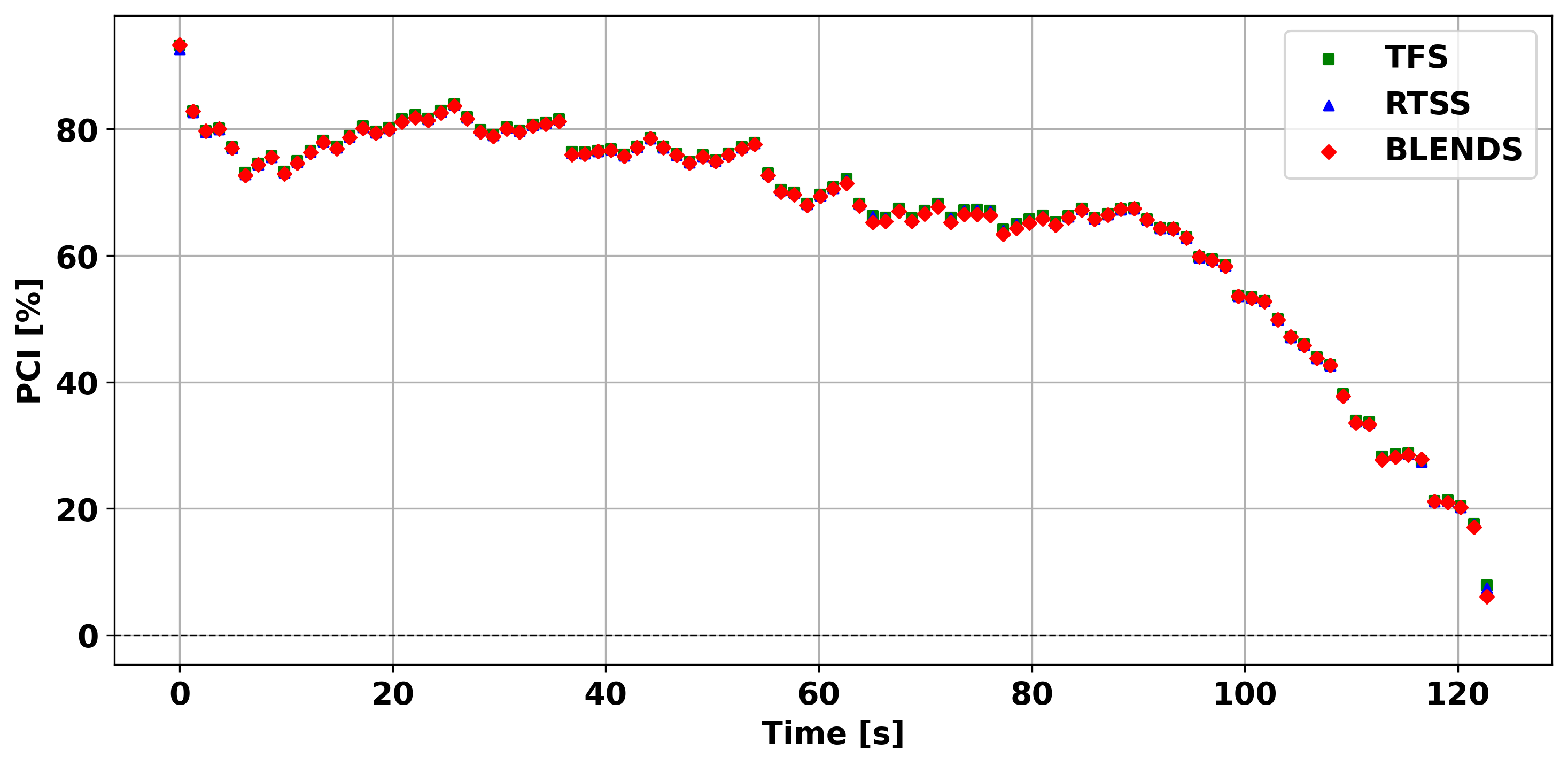}
        \caption{MR-Tes (Traj.~2).}
    \end{subfigure}
    \vspace{0.5em}
    \begin{subfigure}{\columnwidth}
        \includegraphics[width=\columnwidth]{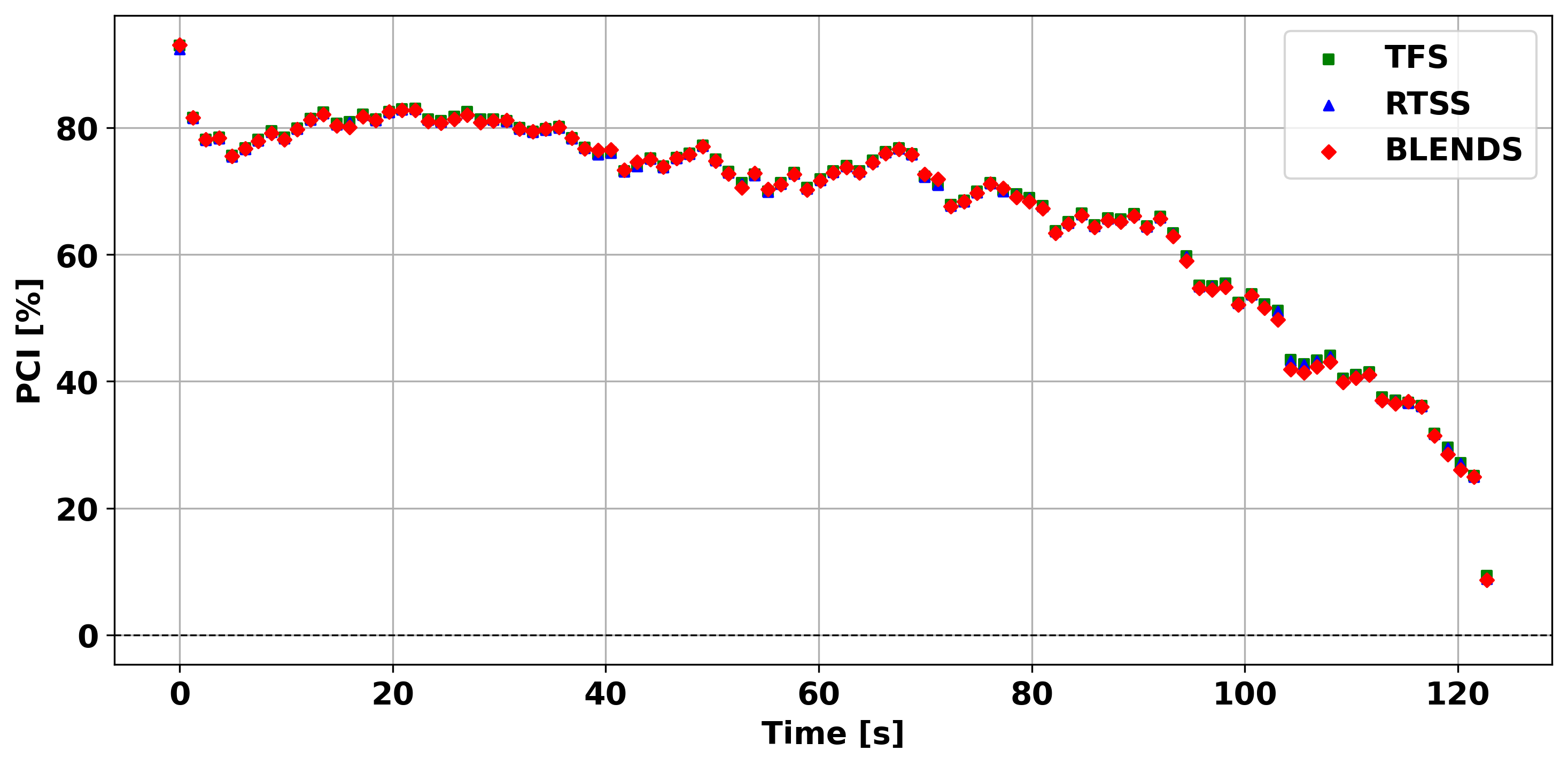}
        \caption{MR-Tes (Traj.~3).}
    \end{subfigure}
    \caption{PCI over time for 
    TFS, RTSS, and BLENDS relative to the forward EKF on the 
    mobile robot dataset test trajectories.}
    \label{fig:arazim_pci}
\end{figure}
\section{Conclusion}\label{sec:conclusion}
\noindent
This paper presented BLENDS, a Bayesian learning-enhanced deep smoothing framework for inertial-aided navigation. BLENDS augments the classical two-filter 
smoother with a transformer-based neural network. Operating as a software-only post-processing framework, BLENDS addresses a fundamental limitation of conventional smoothing algorithms, namely their inability to correct systematic position biases inherited from the forward filter. It achieves this by learning to modify the filter covariances and apply an additive correction directly within the Bayesian smoothing framework.
\\ \noindent
The proposed method was evaluated on two real-world datasets covering 
substantially different platforms and dynamics, a mobile robot and a quadrotor.
Across all unseen test trajectories, BLENDS consistently achieved 
substantial reductions in horizontal position error, with improvements 
of up to 63\% over the forward EKF in the north direction and more 
than 40\% in the east direction, while the classical smoothers RTSS 
and TFS provided negligible position improvement. The $2\sigma$ covariance envelopes confirm that the BLENDS smoothed 
estimates remain statistically consistent throughout all trajectories, 
and the PCI results demonstrated that 
BLENDS achieves a higher reduction in estimation uncertainty than the 
classical smoothers across both platforms.
\\ \noindent
A key design principle of BLENDS is its adherence to the theoretical 
foundations of Bayesian estimation. The BCL jointly supervises the smoothed mean and covariance, enforcing 
minimum-variance estimates while maintaining a tight and honest 
uncertainty representation. The curriculum ramp strategy further 
constrains the learned correction to physically meaningful ranges 
throughout training, preventing the network from producing 
unrealistic outputs. Nevertheless, BLENDS has several limitations. The correction bounds 
and window size must be manually tuned for each platform, requiring 
prior knowledge of the expected dynamic ranges of the navigation 
states. Additionally, similar to model-based approaches, the performance of 
BLENDS is inherently limited by state observability, as unobservable 
states cannot be meaningfully corrected regardless of the learned 
output.
\\ \noindent
The cross-platform generalization demonstrated in this work is achieved with a single network architecture and training procedure, adapted only through platform-specific correction bounds. These results suggest that BLENDS captures a systematic bias inherent to loosely coupled INS/GNSS navigation rather than specific to any particular platform or trajectory. This makes BLENDS a practical and deployable 
software-only post-processing tool for improving navigation accuracy 
in applications where RTK-level precision is desired but only 
standard GNSS measurements are available during operation.
\begin{figure}[!h]
    \centering
    \begin{subfigure}{\columnwidth}
        \includegraphics[width=\columnwidth]{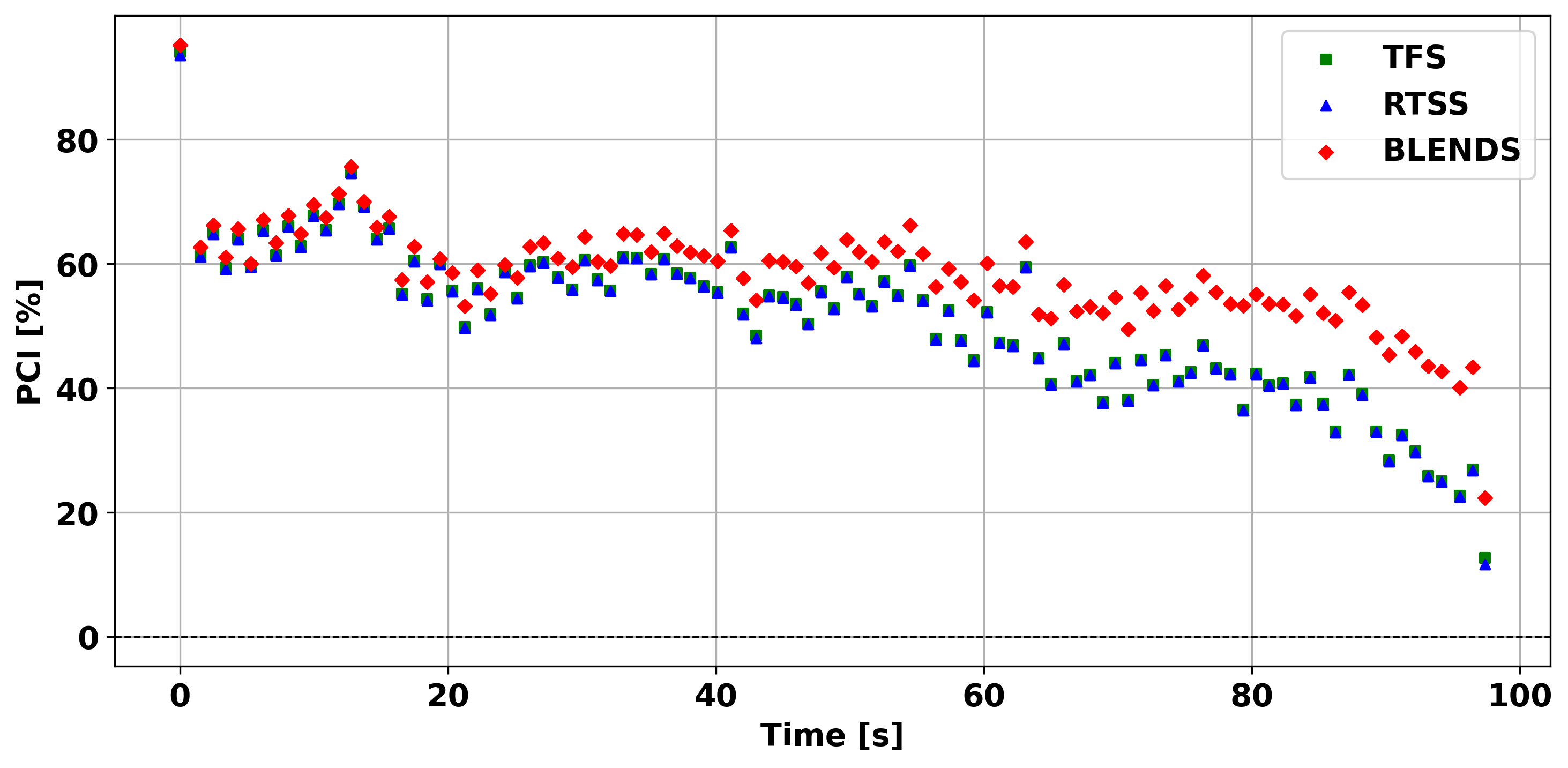}
        \caption{QR-Tes (Traj.~12).}
    \end{subfigure}
    \vspace{0.5em}
    \begin{subfigure}{\columnwidth}
        \includegraphics[width=\columnwidth]{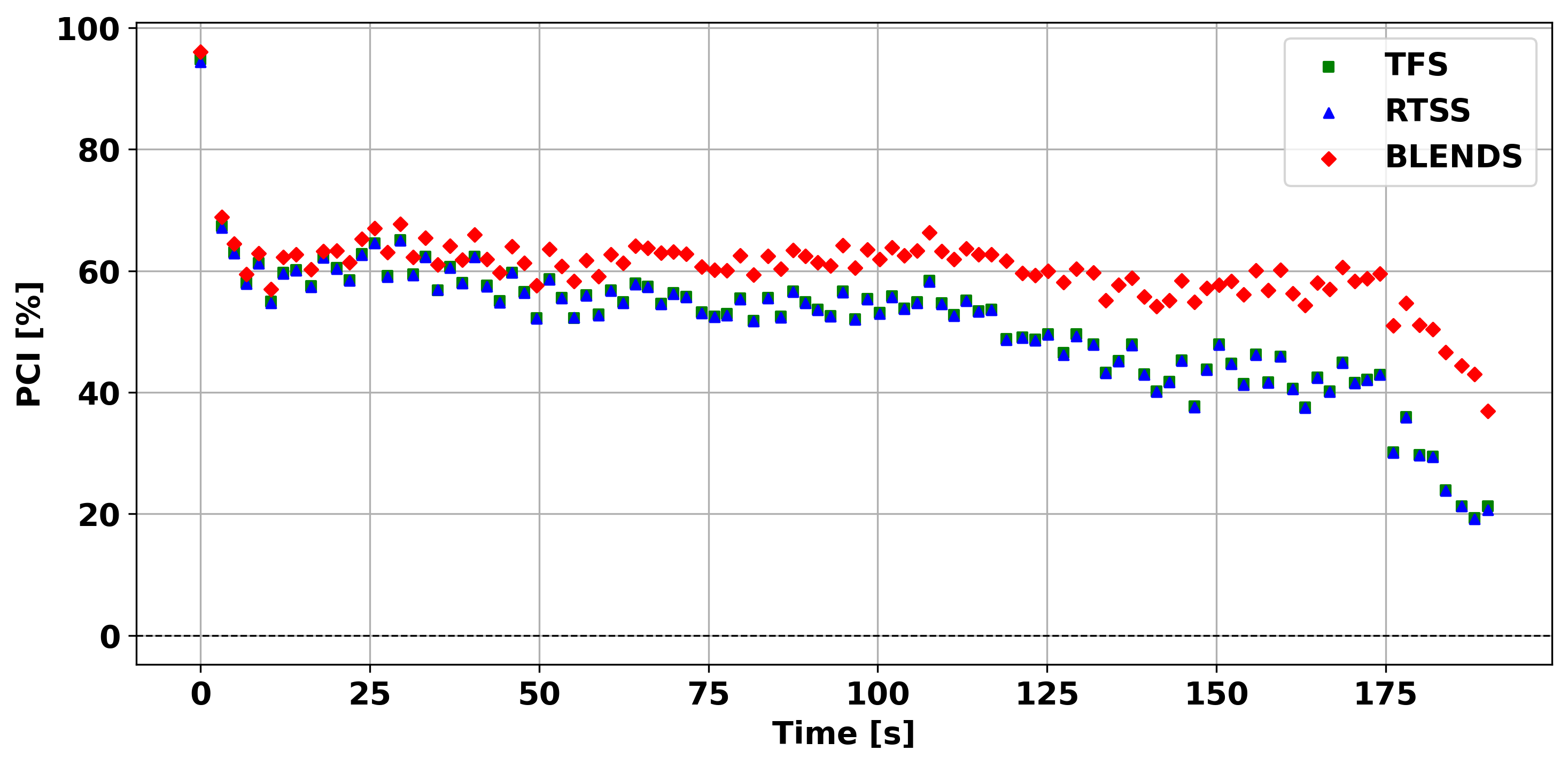}
        \caption{QR-Tes (Traj.~15).}
    \end{subfigure}
    \caption{PCI over time for 
    TFS, RTSS, and BLENDS relative to the forward EKF on the 
    quadrotor dataset test trajectories.}
    \label{fig:mars_pci}
\end{figure}
\section*{Acknowledgment}
\noindent
N.C. is supported by the Maurice Hatter Foundation and the University of Haifa presidential scholarship for outstanding students on a direct Ph.D. track. The authors would like to thank Arazim Ltd. for their support and for providing the Arazim 
Exiguo\textsuperscript{\textregistered} EX-300 unit used in the data collection.

\bibliographystyle{IEEEtran}
\bibliography{bio.bib}

\end{document}